\documentclass[accepted]{uai2026} 
                        
\usepackage[american]{babel}

\usepackage{natbib} 
\usepackage{mathtools} 
\usepackage{booktabs} 
\usepackage{tikz} 
\usepackage{multirow}
\usepackage{amsmath}
\usepackage{amssymb}
\usepackage{algorithm}
\usepackage{algorithmic}
\usepackage{amsthm}
\usepackage{subcaption}
\usepackage{xcolor}
\newcommand{\Myfrac}[2]{\ensuremath{#1\mathord{\left/\right.\kern-\nulldelimiterspace}#2}}
\newcommand{\rev}[1]{{\color{black}#1}}

\title{Efficient Federated Conformal Prediction with Group-Conditional Guarantee}

\author[1]{Haifeng~Wen}
\author[2]{Osvaldo~Simeone}
\author[1,3]{Hong~Xing}
\affil[1]{%
    IoT Thrust\\
    The Hong Kong University of Science and Technology (Guangzhou)\\
    Guangzhou, China
}
\affil[2]{%
    Institute for Intelligent Networked Systems (INSI)\\
    Northeastern University London\\
    London, UK
}
\affil[3]{%
    Department of ECE\\
    The Hong Kong University of Science and Technology\\
    HK SAR
}

\theoremstyle{plain}
\newtheorem{theorem}{Theorem}[section]

\newtheorem{corollary}{Corollary}[section]
\newtheorem{lemma}{Lemma}[section]

\theoremstyle{definition}

\newtheorem{claim}{Claim}
\theoremstyle{remark}

\begin{document}
\maketitle

\begin{abstract}
   
Deploying trustworthy AI systems requires principled uncertainty quantification. Conformal prediction (CP) is a widely used framework for constructing prediction sets with distribution-free coverage guarantees. In many practical settings, including healthcare, finance, and mobile sensing, the calibration data required for CP are distributed across multiple clients, each with its own local data distribution. In this federated setting, data can often be partitioned into, potentially overlapping, groups, which may reflect client-specific strata or cross-cutting attributes such as demographic or semantic categories. \rev{We propose \emph{group-conditional} federated conformal prediction (GC-FCP), a federated extension of conditional conformal calibration for a target mixture over prespecified groups.} \rev{GC-FCP constructs mergeable, atom-stratified coresets from local calibration scores, enabling compact aggregation at the server when the number of active atoms is moderate.} Experiments on synthetic and real-world datasets validate the performance of GC-FCP compared to centralized calibration baselines.
\rev{The code of our work can be found at \url{https://github.com/HaifengWen/GC-FCP}.}
\end{abstract}

\section{Introduction} \label{sec:introduction}

\begin{figure*}[t] 
\centering
    \includegraphics[width=1\linewidth]{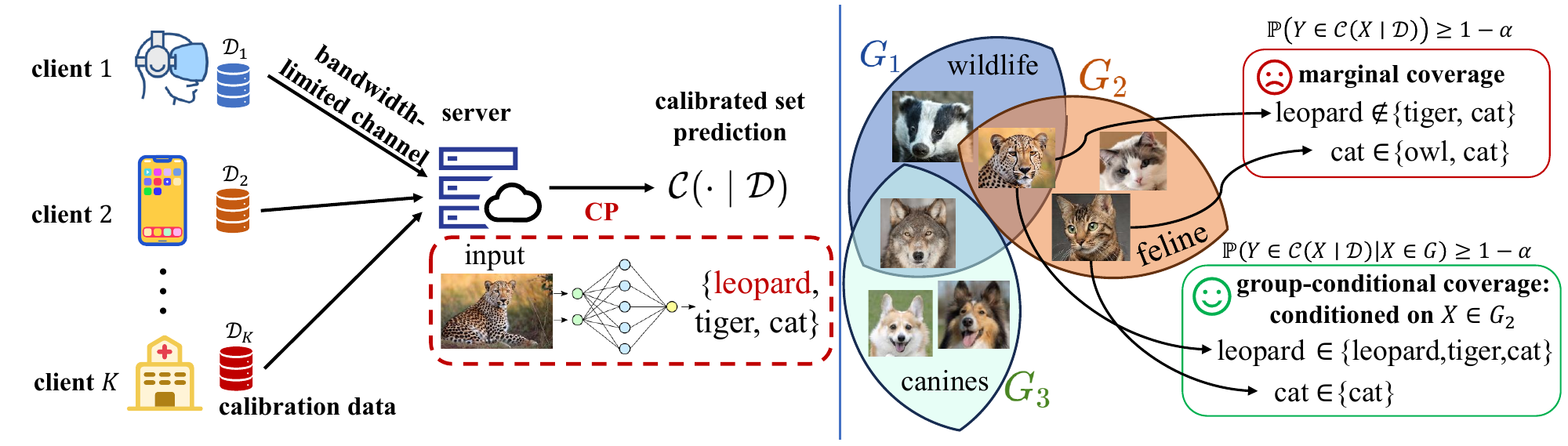}
    \caption{\emph{Left}: In a federated system with heterogeneous clients, each client $k$ holds local calibration data $\mathcal D_k$ and communicates over a bandwidth-limited channel to a central server. The server aggregates these summaries to perform conformal calibration (CP) and outputs a set-valued predictor $\mathcal{C}(\cdot\mid\mathcal D)$. 
    \emph{Right}: GC-FCP targets group-conditional coverage for potentially overlapping groups $\mathcal{G}=\{G_1,G_2,G_3\}$, ensuring the inequality $\mathbb P(Y\in \mathcal{C}(X\mid\mathcal D)\mid X\in G)\ge 1-\alpha$ for all groups $G \in \mathcal{G}$, whereas methods that only guarantee marginal coverage $\mathbb P(Y\in \mathcal{C}(X\mid\mathcal D))\ge 1-\alpha$, such as FCP~\citep{lu2023federated}, may still under-cover within specific groups.}
\label{fig:system_model}
\end{figure*}

Deploying trustworthy AI systems critically depends on uncertainty quantification to enable reliability control at deployment time.
Given a pretrained model, \emph{conformal prediction (CP)} post-processes the model's outputs to construct prediction sets with finite-sample, distribution-free coverage guarantees.
The \emph{calibration} of the prediction set is carried out offline by leveraging held-out calibration data \citep{vovk2005algorithmic,angelopoulos2021gentle,barber2021jackknifeplus,simeone2025uncertainty}.
In many practical deployments, however, calibration data are inherently distributed and subject to privacy constraints, so that each client must retain its data on-site, e.g., hospitals, banks, or Internet-of-Things devices \citep{mcmahan2017communication}.

\emph{Federated conformal prediction} (FCP)~\citep{lu2023federated} addresses the outlined distributed setting with the goal of preserving marginal coverage with respect to a distribution obtained by mixing local data distributions. 
Pursuing a similar marginal coverage guarantee, \citet{humbert2023one} and \citet{humbert2024marginal} proposed one-shot FCP schemes based on a quantile-of-quantiles estimator.
Beyond calibration targeting the mixture distribution, FCP-Pro~\citep{li2025fcp} and personalized FCP~\citep{min2025personalized} address the heterogeneity of local data distributions by training additional models during the calibration process, providing marginal coverage on local distributions or on the distribution of a new client.
To address robustness, \citet{kang2024certifiably} studied Byzantine clients that may report arbitrary statistics.
Considering more general connectivity constraints underlying communications, decentralized calibration via message passing over arbitrary graphs was proposed by \citep{wen2025distributed}, and distributed remote calibration protocols were studied in \citep{zhu2025conformal} for wireless sensor networks.

The state of the art on federated calibration summarized above did not target \emph{group-conditional} coverage guarantees. 
In federated settings, data can often be partitioned into, potentially overlapping, groups, which may reflect client-specific strata or cross-cutting attributes such as demographic or semantic categories (see Fig.~\ref{fig:system_model}).
For centralized settings, Mondrian CP~\citep{vovk2003mondrian} addresses group-conditional coverage over disjoint groups of covariates, while CondCP~\citep{gibbs2023conformal} allows for overlapping groups by reframing conditional coverage as simultaneous coverage over a class of covariate shifts.
Kandinsky CP~\citep{pmlr-v267-bairaktari25a} extended Mondrian CP and CondCP to groups that depend on both covariates and labels.
Other related research directions for centralized scenarios include localized coverage guarantees~\citep{guan2023localized}; learning improved conformity scores to reduce conditional miscoverage~\citep{xie2024boosted}; as well as analyses of sample-conditional validity for split conformal methods \citep{duchi2025sample}.

All this prior art on group-conditional CP has focused on centralized settings.
To address this knowledge gap, in this paper, we propose \rev{a federated, atom-stratified extension of CondCP~\citep{gibbs2023conformal}, termed \emph{group-conditional federated conformal prediction (GC-FCP)}, which enables communication-efficient and computation-efficient prediction-set construction under mild conditions, such as when the overlap between groups is not too severe.}
Unlike FedCF~\citep{srinivasan2025fedcf}, GC-FCP does not treat different groups separately, seeking simultaneous coverage across all groups. 

The main contributions are summarized as follows:
\begin{itemize}
    \item \rev{We first propose \emph{centralized GC-FCP}, an extension of CondCP~\citep{gibbs2023conformal} that achieves group-conditional coverage guarantees under a federated target mixture distribution.}
    \item \rev{We develop \emph{GC-FCP}, a federated protocol that compresses atom-stratified calibration scores into mergeable T-Digest~\citep{dunning2019computing} coresets, avoiding the loss of group-membership information caused by a single global sketch and double counting caused by overlapping group sketches.}
    \item We establish group-conditional coverage bounds for GC-FCP, making it explicit how the coreset compression level affects the achieved coverage.
    \item We validate the proposed methods on synthetic and real-world benchmarks, demonstrating that GC-FCP attains group-conditional reliability while substantially reducing computational overhead.
\end{itemize}

\section{Problem Setting}
\label{sec:problem_setting}

As illustrated in Fig.~\ref{fig:system_model}, we study a federated calibration setting that follows reference \citep{lu2023federated}. Accordingly, we consider a network consisting of $K$ clients that communicate with a central server. 
In this setup, each client $k\in\{1,\dots,K\}\triangleq [K]$ has access to a local calibration dataset
$
\mathcal{D}_k = \left\{(X_{i,k},Y_{i,k})\right\}_{i=1}^{n_k},
$
with data points $(X_{i,k},Y_{i,k})$ drawn i.i.d.\ from a client-specific distribution $P_k$ over $\mathcal{X}\times\mathcal{Y}$. 
We define a global calibration dataset as the union
$
\mathcal{D} = \bigcup_{k=1}^K \mathcal{D}_k,
$ with
$n = \sum_{k=1}^K n_k$.
As in \citep{lu2023federated}, the goal is to calibrate a shared pre-trained model $f:\mathcal{X}\mapsto\mathcal{Y}$ through communication with the central server.
Calibration is evaluated on test data $(X_{n+1},Y_{n+1})$ drawn from the mixture 
\begin{equation} \label{eq:mixture_model}
    \mathrm{P}= \sum_{k=1}^{K} \pi_k \mathrm{P}_k,
\end{equation}
for some arbitrary and known mixture weights $\pi_k \ge 0$ with $\sum_{k=1}^{K} \pi_k = 1$.
In practice, the weights $\{\pi_k\}_{k=1}^{K}$ dictate the relative relevance of the data of each client $k$ to the calibration of model $f$.
The vector $\pi$ should be interpreted as part of the deployment target, e.g., equal client, population, or policy-driven weighting, and it does not necessarily imply the empirical client-sample-size mixture.

Calibration aims at obtaining a set predictor
\begin{equation} \label{eq:abs_prediction_set}
\mathcal{C}(\cdot \mid \mathcal{D}) : \mathcal{X}\rightarrow 2^{\mathcal{Y}}.
\end{equation}
using the distributed calibration data via clients-to-server communication. 
Prior work \citep{lu2023federated} imposed the constraint that the prediction set \eqref{eq:abs_prediction_set} satisfies the marginal coverage condition 
\begin{equation} \label{eq:marginal_coverage}
    \mathbb{P}(Y_{n+1} \in \mathcal{C}(X_{n+1}\mid \mathcal{D})) \ge 1-\alpha
\end{equation}
with respect to the mixture distribution \eqref{eq:mixture_model}. 
In contrast, as explained next, we impose a more flexible conditional coverage condition based on grouping.

Let $\mathcal{G} \subseteq 2^{\mathcal{X}}$ be a finite collection of groups in the covariate space $\mathcal{X}$ and $\bigcup_{G\in \mathcal{G}}=\mathcal{X}$. Importantly, the groups are generally overlapping, i.e., there exist groups $G_i,G_j\in \mathcal{G}$ such that $G_i \bigcap G_j \neq \emptyset $ (see Fig.~\ref{fig:system_model}).
Groups may or may not represent client-specific partitions. 
In fact, some groups $G$ may correspond to data categories that are more representative of the distribution $\mathrm{P}_k$ of a given client $k$. 
However, it may also be that groups refer to categories that apply equally across all clients. We consider the set $\mathcal{G}$ to be arbitrary and given.

Given a target mis-coverage level $\alpha\in(0,1)$, calibration aims to construct set prediction $\mathcal{C}(\cdot\mid \mathcal{D})$ such that,
for every group $G\in\mathcal{G}$, the following group-conditional coverage requirement holds:
\begin{equation}
\mathbb{P}\left( Y_{n+1}\in \mathcal{C}(X_{n+1}\mid \mathcal{D}) \,\middle|\, X_{n+1}\in G \right)\ge 1-\alpha,
\label{eq:group_conditional_coverage}
\end{equation}
where the probability in \eqref{eq:group_conditional_coverage} is taken over the calibration data in $\mathcal{D}$ and the test data $(X_{n+1},Y_{n+1})$.
For condition \eqref{eq:group_conditional_coverage} to be meaningful, we assume that all groups $G\in \mathcal{G}$ satisfy the inequality $\mathbb{P}(X_{n+1} \in G) >0$ under the mixture distribution \eqref{eq:mixture_model}.
Note that this does not require that the inequality $\mathbb{P}_k(X\in G)>0$ holds for every client $k$. Rather, it only requires that at least one component with probability $\pi_k>0$ places positive mass on the group $G$.

{\color{black}
The guarantee \eqref{eq:group_conditional_coverage} applies to a prespecified finite group family.
This may include public common groups for which all clients can compute the same membership vector.
Private groups, as well as unknown or emergent groups, requiring separate discovery or auditing procedures, are outside the scope of this finite-group guarantee.
}

If the groups $\mathcal{G}$ were disjoint, the conditional coverage requirement \eqref{eq:group_conditional_coverage} could be obtained by combining FCP \citep{lu2023federated} and  Mondrian CP \citep{vovk2003mondrian}. 
Accordingly, one applies the FCP protocol separately for each group $G \in \mathcal{G}$.
However, ensuring the conditions \eqref{eq:group_conditional_coverage} become difficult when groups overlap, and we address this challenge in this paper.


\section{Preliminaries}
\label{sec:preliminaries}
To start, consider, for reference, a centralized split-conformal setting in which a server has access to i.i.d.\ calibration data $\{(X_i,Y_i)\}_{i=1}^{n}$ and a test covariate $X_{n+1}$. 
In this section, we review CondCP \citep{gibbs2023conformal}, which addresses the problem of satisfying the constraints \eqref{eq:group_conditional_coverage} in such a centralized setting. 

Let $s:\mathcal{X}\times\mathcal{Y}\mapsto\mathbb{R}$ be a score function derived from a predictive model $f$, and
define the calibration scores $S_i = s(X_i,Y_i)$ with $i=1,\dots,n$.
Furthermore, let $\ell_\alpha(\theta,S)$ denote the pinball loss at level $1-\alpha$:
\begin{equation} \label{eq:pinball_loss}
\ell_\alpha(\theta,S) =
\begin{cases}
(1-\alpha)(S-\theta), & S \ge \theta,\\
\alpha(\theta-S), & S < \theta.
\end{cases}
\end{equation}

Given the set of groups $\mathcal{G}$ (possibly overlapping), define the group-membership map $\Phi:\mathcal{X} \to \{0,1\}^{|\mathcal{G}|}$ as 
\begin{equation} \label{eq:membership_vector}
    \Phi(x) = \big(\mathbf{1}\{x \in G\}\big)_{G \in \mathcal{G}} \in \{0,1\}^{|\mathcal{G}|},
\end{equation}
which encodes membership of element $x \in \mathcal{X}$ belonging to each group of $\mathcal{G}$ as a binary indicator vector (where $\mathbf{1}\{\text{true}\}=1$ and $\mathbf{1}\{\text{false}\}=0$).
Then, CondCP defines the set of functions
\begin{equation}
\mathcal{F}_{\mathcal{G}}
=
\left\{
x \mapsto \beta^\top \Phi(x): \beta\in\mathbb{R}^{|\mathcal{G}|}
\right\},
\label{eq:group_linear_class}
\end{equation}
which corresponds to arbitrary linear combinations for the group indicators. 
Specifically, for any input test score $S\in\mathbb{R}$, CondCP solves the \emph{augmented quantile regression} problem:
\begin{multline}
\hat{g}_S
=
\arg\min_{g\in\mathcal{F}_{\mathcal{G}}} \
\left\{ \frac{1}{n+1}\sum_{i=1}^{n}\ell_\alpha\bigl(g(X_i),S_i\bigr) \right. \\ 
\left. +
\frac{1}{n+1}\ell_\alpha\bigl(g(X_{n+1}),S\bigr) \right\} .
\label{eq:condcp_augmented_qr}
\end{multline}
Note that the solution $\hat{g}_S$ is a function in class \eqref{eq:group_linear_class}. By the properties of the pinball loss, the solution $\hat{g}_S$ represents a covariate ($X_{n+1}$)-dependent version of the $(1-\alpha)$-empirical quantile of the augmented calibration scores $\{S_i\}_{i=1}^{n} \cup \{S\}$ \citep{gibbs2023conformal}.
The prediction set is then defined as the set of all labels $y \in \mathcal{Y}$ whose score $s(X_{n+1},y)$ is no larger than the corresponding empirical quantile $\hat{g}_{s(X_{n+1},y)}(X_{n+1})$, i.e.,
\begin{multline}
    \mathcal{C}(X_{n+1} \mid \mathcal{D}) \\ 
=
\left\{
y\in\mathcal{Y}:
s(X_{n+1},y)\le \hat{g}_{\,s(X_{n+1},y)}(X_{n+1})
\right\}.
\label{eq:condcp_set}
\end{multline}

The prediction set $\eqref{eq:condcp_set}$ is constructed by addressing an equivalent formulation of the convex problem \eqref{eq:condcp_augmented_qr} via dual methods (see Appendix~\ref{sec:dual-construction} for details). 
Under the assumption of i.i.d. calibration and test data, the set $\mathcal{C}(X_{n+1} \mid \mathcal{D})$ in \eqref{eq:condcp_set} satisfies group-conditional coverage \eqref{eq:group_conditional_coverage}.

\textbf{Communication overhead:} Solving problem~\eqref{eq:condcp_augmented_qr} at the server would require collecting all calibration scores and their group-membership features, yielding a communication overhead $\mathcal{O}(n)$.

\textbf{Computational overhead:} Solving problem \eqref{eq:condcp_augmented_qr} requires computational complexity of order $\mathcal{O}(n^{3/2}|\mathcal{G}|^2)$ using standard convex problem solvers~\citep{renegar1988newton,nesterov1994interior}.

\begin{table}[t]
\centering
\caption{\rev{Communication load and computational complexity. For GC-FCP, $\mathcal A_k^+$ is the set of non-empty atoms observed by client $k$, and $\mathcal A^+=\bigcup_k\mathcal A_k^+$.}
}   
\label{tab:complexity}
{\color{black}
\begin{tabular}{ccc}
\hline
\textbf{Method} & \textbf{Comm.} & \textbf{Comp.} \\
\hline
\begin{tabular}[c]{@{}c@{}}CondCP \& \\ Centralized \\ GC-FCP \end{tabular} &
$\mathcal{O}(n)$ &
$\mathcal{O}\left(n^{3/2}|\mathcal{G}|^2\right)$ \\
\hline
GC-FCP &
$\mathcal{O}\big(\delta\sum_{k}|\mathcal A_k^+|\big)$ &
$\mathcal{O}\big((\delta|\mathcal A^+|)^{3/2}|\mathcal{G}|^2\big)$ \\
\hline
\end{tabular}
}
\end{table}

\section{Group-Conditional Federated Conformal Prediction (GC-FCP)}
\label{sec:method}

This section introduces \emph{group-conditional federated conformal prediction} (GC-FCP), a federated calibration procedure designed to achieve group-conditional coverage \eqref{eq:group_conditional_coverage} over a prescribed, generally overlapping, collection of groups $\mathcal{G}$ under any mixture distribution \eqref{eq:mixture_model}.
In GC-FCP, each client computes calibration scores locally and communicates to the server only a compact, group-stratified summary, while the server efficiently constructs the CP set for each test point without accessing raw client-level calibration scores.

\subsection{Federated Augmented Quantile Regression}
\label{sec:centralized_gcfcp}
To start, consider an ideal setting in which the entire calibration dataset $\mathcal{D}$, which includes the datasets $\mathcal{D}_k$ from all clients $k \in [K]$, is available at the central server. Assume also no computational limitations at the server. 
Even in this simplified setup, as explained in the previous section, CondCP would fail to guarantee the conditional coverage requirements \eqref{eq:group_conditional_coverage}, since the data points in dataset $\mathcal{D}$ are not i.i.d. with respect to the mixture distribution \eqref{eq:mixture_model}.
We address this challenge in this section by proposing \emph{centralized GC-FCP}, a centralized calibration scheme that guarantees the desired condition \eqref{eq:group_conditional_coverage}.

As in Section~\ref{sec:preliminaries}, let $s:\mathcal{X}\times\mathcal{Y}\mapsto\mathbb{R}$ be a score function derived from the shared predictive model $f$.
We denote the calibration scores at client $k$ as $S_{i,k} = s(X_{i,k},Y_{i,k})$ with $ \ i\in\{1,\dots,n_k\}$ and $k\in[K]$.
Furthermore, given a test point $(X_{n+1},Y_{n+1}) \sim \mathrm{P}$, we denote its score by $S_{n+1} = s(X_{n+1},Y_{n+1})$.

For a test covariate $X_{n+1}$ and an input score value $S\in\mathbb{R}$, centralized GC-FCP solves the following
\emph{federated} augmented quantile regression estimator as a generalization of the CondCP problem \eqref{eq:condcp_augmented_qr}:
\begin{multline}
\hat{g}_{S}
=
\underset{g\in\mathcal{F}_{\mathcal{G}}}{\operatorname*{argmin}}
\sum_{k=1}^{K}\frac{\pi_k }{n_k+1}
\bigg(
\sum_{i=1}^{n_k}\ell_\alpha(g(X_{i,k}),S_{i,k}) \\ 
 +
\ell_\alpha(g(X_{n+1}),S)
\bigg),
\label{eq:gcfcp_augqr}
\end{multline}
where the function set $\mathcal{F}_{\mathcal{G}}$ is defined in \eqref{eq:group_linear_class}.
The objective in problem \eqref{eq:gcfcp_augqr} incorporates the scores of all $K$ clients, with each $k$-th term weighted by the factor $\pi_k/(n_k+1)$, thus matching the structure of the target mixture distribution \eqref{eq:mixture_model}.
Using this definition, centralized GC-FCP obtains the prediction set as in \eqref{eq:condcp_set}.

\begin{theorem}[Group-conditional coverage for centralized GC-FCP]
\label{thm:gcfcp_exact_group}
For every group $G\in\mathcal{G}$, the set \eqref{eq:condcp_set} with the empirical quantile \eqref{eq:gcfcp_augqr} produced by centralized GC-FCP satisfies the conditional coverage condition \eqref{eq:group_conditional_coverage}.
\end{theorem}
\textit{Proof:} See Appendix \ref{sec:proof-thm-exact-group}.

As summarized in Table~\ref{tab:complexity}, the centralized GC-FCP shares the same communication and computational overheads as CondCP.

\subsection{Group-Conditional Federated Conformal Prediction}
\label{sec:gc-fcp}

The previous subsection focused on a centralized version of GC-FCP, which solves problem \eqref{eq:gcfcp_augqr} using the calibration scores pooled from all clients.
To mitigate the communication and computational overhead associated with this approach, we now introduce GC-FCP, a distributed protocol that replaces the full calibration dataset with mergeable atom-based coresets.
As summarized in Table~\ref{tab:complexity}, this protocol reduces the communication load and enables a more efficient server-side optimization in conditions to be elaborated.

GC-FCP starts by applying a coreset construction based on T-Digest \citep{dunning2019computing}, which is a sketching algorithm for computing approximations of quantiles, and is elaborated in the sequel. 
The sketches produced by T-Digest are then used to solve the optimization problem \eqref{eq:gcfcp_augqr} at the central server. 
While T-Digest was also used by FCP \citep{lu2023federated}, GC-FCP must additionally account for the structure of the groups in the set $\mathcal{G}$, producing stratified structures (see Fig.~\ref{fig:atom_illustration}).

\subsubsection{T-Digest}\label{subsubsec:tdigest}
Let $\{(R_i,w_i)\}_{i=1}^\ell$ be weighted real-valued samples with $R_i\in\mathbb{R}$ and weights $w_i>0$, and define the total weight as $W=\sum_{i=1}^\ell w_i$. The associated weighted empirical cumulative distribution function (CDF) is
$
F(t)=\frac{1}{W}\sum_{i=1}^\ell w_i\,\mathbf{1}\{R_i\le t\},
$
with generalized empirical quantile function $Q(u)=\inf\{t:F(t)\ge u\}$ for $u\in[0,1]$.

\textbf{T-Digest:}
A \emph{T-Digest} \citep{dunning2019computing} produces an ordered collection of $m < \ell$ clusters $\mathcal{R}_c\subseteq\{1, \ldots, \ell\}$, i.e., if $c_1<c_2$, $R_i \le R_j$ for any $i \in \mathcal{R}_{c_1}$ and $j\in \mathcal{R}_{c_2}$. Cluster $\mathcal{R}_c$ is described by a cluster representative $\bar R_c$ and an aggregate weight $W_c$.
Each cluster $\mathcal{R}_c$ includes a subset of samples $\{R_i\}_{i=1}^\ell$ with total weight
$
W_c = \sum_{i\in \mathcal{R}_c}w_i,
$
and the weighted cluster means
$
\bar R_c = \frac{1}{W_c}\sum_{i\in \mathcal{R}_c}w_i R_i.
$
In summary, T-Digest returns the weights and the cluster means as
\begin{equation}
\mathsf{TD}=\big\{(\bar R_c,W_c)\big\}_{c=1}^{m}.
\end{equation}

It is worth noting that T-Digest enforces the ordering $\bar{R}_1 \le \bar{R}_2 \le \dots \le \bar{R}_{m}$, producing ordered cluster means.
To this end, it orders the original set $\{R_i\}_{i=1}^{\ell}$ such that we have $R_1 \le R_2 \le \dots\le R_\ell$. 
Clusters are then constructed greedily as follows. 
To elaborate, let $V_c=\sum_{i=1}^c W_i$ with $V_0=0$, and define the empirical left and right quantile boundaries of cluster $c \in \{1,\ldots,m \}$ as
\begin{equation}
q_c^{\mathrm{L}}=\frac{V_{c-1}}{W}\ \text{ and }\ q_c^{\mathrm{R}}=\frac{V_c}{W}.
\end{equation}
Samples $R_1,R_2,\ldots,R_\ell$ are aggregated in an increasing order into the current cluster $\mathcal{R}_c$ as long as adding the next sample preserves the constraint
\begin{equation}\label{eq:td_size_criterion}
r\big(q_c^{\mathrm{R}}\big)-r\big(q_c^{\mathrm{L}}\big)\le 1,\qquad c=1,\ldots,m.
\end{equation}
with scale function
\begin{equation}\label{eq:td_scale}
r(q)=\frac{\delta}{2\pi}\arcsin(2q-1).
\end{equation}
for some parameter $\delta > 0$ that controls the level of compression.

Intuitively, the condition \eqref{eq:td_size_criterion} enforces finer resolution near the tails of the empirical distribution of the scalars $\{R_i\}_{i=1}^{\ell}$, while permitting larger clusters near the median. 
By this construction, the number of retained clusters scales as $m =\Theta(\delta)$ \citep{dunning2019computing}.

 From the digest $\mathsf{TD}=\big\{(\bar R_c,W_c)\big\}_{c=1}^{m}$, the approximate quantile function for $u\in[0,1]$ is given by 
 \begin{equation} \label{eq:tdigest_approximate_quantile}
    \widehat Q(u)=\inf\{t:\widehat F(t)\ge u\},
\end{equation}
where
 \begin{equation} \label{eq:td_approximate_cdf}
    \widehat F(t)\ =\ \frac{1}{W}\sum_{c=1}^{m} W_c\,\mathbf 1\{\bar R_c\le t\}.
\end{equation} 
is an estimate of the empirical CDF $F(t)$.
The quality of this estimate will be analyzed in Section~\ref{sec:theoretical_guarantee}.

\textbf{Merging T-Digests:}
A key property of T-Digest is that digests can be merged and then re-compressed such that the number of retained clusters
remains $\Theta(\delta)$ \citep{dunning2019computing}.
Let
$\mathsf{TD}^{(j)}=\{(\bar R^{(j)}_c,W^{(j)}_c)\}_{c=1}^{m_j}$
be digests obtained from disjoint datasets indexed by $j = 1,\ldots, J$.
To merge them, one applies the same procedure discussed above to the pooled samples $\bigcup_{j=1}^{J} \mathsf{TD}^{(j)} $. We denote the merged digest as
$\mathrm{Merge}(\mathsf{TD}^{(1)},\ldots,\mathsf{TD}^{(J)})$.

\subsubsection{GC-FCP} \label{subsubsec:gc-fcp}

\begin{figure}[t]
    \centering
    \includegraphics[width=\linewidth]{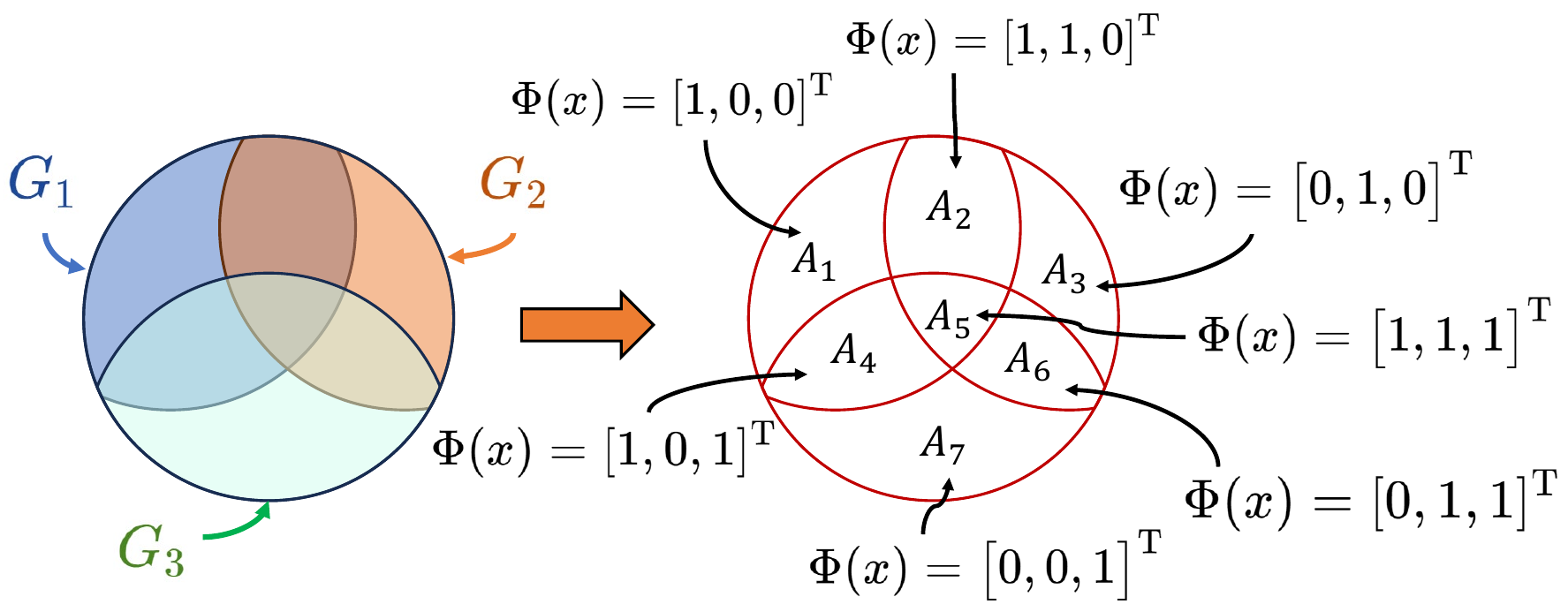}
    \caption{Illustration of the atom partitions applied by GC-FCP. Given the set of overlapping groups $\mathcal G=\{G_1,G_2,G_3\}$, the resulting $7$ non-empty atoms $\mathcal{A}=\{A_1,\ldots,A_7\}$ are shown on the right, together with the corresponding group-membership vector \eqref{eq:membership_vector}.}
    \label{fig:atom_illustration}
\end{figure}

In order to address the group-conditional constraint \eqref{eq:group_conditional_coverage}, GC-FCP first stratifies the data at the clients into \emph{disjoint atoms} based on unique group membership patterns. Independent T-Digests are then constructed per atom, and used at the server as coresets to approximate the solution of the quantile regression problem \eqref{eq:gcfcp_augqr}. 

\textbf{Atoms construction:}
As illustrated in Fig.~\ref{fig:atom_illustration}, GC-FCP first partitions the input domain $\mathcal{X}$ into a collection $\mathcal{A}=\{A\}_{A\in \mathcal{A}}$ of atoms, which are defined as $\bigcup_{A\in \mathcal{A}}A=\mathcal{X}$ and $A\bigcap A^\prime = \emptyset$, for any pair $A,A^\prime \in \mathcal{A}$. 
The collection $\mathcal{A}$ corresponds to the smallest-cardinality partition of the set $\mathcal{X}$ such that all the sets $\mathcal{G}$ can be recovered via union operations on its atoms. 
Formally, for each subset of groups, $\mathcal S\subseteq \mathcal G$, define 
\begin{equation} \label{eq:atom_definition}
    A(\mathcal S)=\left(\bigcap_{G\in \mathcal S} G\right)\bigcap\left(\bigcap_{G\in \mathcal G\setminus \mathcal S} G^c\right) 
\end{equation}
as the intersection of all groups in subset $\mathcal S$ and the complement of all groups not in subset $\mathcal S$.
The construction $\mathcal{A}=\{A(\mathcal S)\}_{\mathcal S\subseteq\mathcal G}$ yields the collection of atoms \citep{kallenberg2002foundations}.
Although arbitrary overlaps can induce up to $2^{|\mathcal G|}$ atoms, GC-FCP is efficient for many practical group settings.
For example, 
for one-dimension interval groups, the number of atoms scales linearly with $|\mathcal{G}|$.


By construction, each atom $A$ corresponds to all points $x\in\mathcal{X}$ with the same membership vector $\Phi(x)$. 
Accordingly, for any $x\in A$, we can write $\Phi_{A} = \Phi(x)$.
An example is illustrated in Fig.~\ref{fig:atom_illustration}.

\textbf{Local-score partition and digest:}
We now partition the set of local scores at client $k\in[K]$ into $|\mathcal{A}|$ subsets $\mathcal{D}_{k,A}$, one per atom $A$ as
\begin{equation}
\mathcal{D}_{k,A} = \left \{s(X,Y): X \in A,\ (X,Y)\in \mathcal{D}_k \right \}.
\end{equation}
{\color{black}
Define the active local and global atom sets as
\begin{equation}
\mathcal A_k^+=\{A\in\mathcal A:\mathcal D_{k,A}\neq\emptyset\},
\qquad
\mathcal A^+=\bigcup_{k=1}^{K}\mathcal A_k^+,
\end{equation}
which satisfy the inequalities $|\mathcal A^+|\le \min\{2^{|\mathcal G|},n\}$ and $|\mathcal A_k^+|\le \min\{2^{|\mathcal G|},n_k\}$.
Empty atoms produce no digest.
}
\rev{Each client $k\in[K]$ builds a separate digest $\mathsf{TD}_{k,A}$ upon the scores $\mathcal{D}_{k,A}$ for each active atom $A\in\mathcal A_k^+$ with equal weights $w=\pi_k/(n_k+1)$.}
Following Section~\ref{sec:gc-fcp}, the resulting digest $\mathsf{TD}_{k,A}$ consists of a set of means and weights with $m_{k,A}=\Theta(\delta)$ clusters 
\begin{equation}
\mathsf{TD}_{k,A} = \left\{ \left( \bar{S}_{k,A,c},\, W_{k,A,c} \right) \right\}_{c=1}^{m_{k,A}},
\label{eq:local_digest_clusters}
\end{equation}
where
$
    \bar{S}_{k,A,c} = \frac{1}{W_{k,A,c}} \sum_{i \in \mathcal{S}_{k,A,c}} \frac{\pi_k}{n_k+1} S_i 
$
is the weighted mean of the scores in cluster $\mathcal{S}_{k,A,c}$, and
$
    W_{k,A,c} = |\mathcal{S}_{k,A,c}| \frac{\pi_k}{n_k+1}
$
is the aggregated weight.
    
\rev{Next, client $ k \in [K] $ transmits the digest $ \mathsf{TD}_{k,A} $ along with the identifier of active atom $A\in\mathcal A_k^+$ to the server.}
\rev{For each atom $ A \in \mathcal{A}^+ $, the server merges the received digests $ \{\mathsf{TD}_{k,A}\}_{k=1}^K $ to obtain a global digest}
\begin{equation}
\begin{aligned}
\mathsf{TD}_{A} 
& =\mathrm{Merge}\bigl( \mathsf{TD}_{1,A},\ldots,  \mathsf{TD}_{K,A}\bigr) \\ 
& = \left\{ \left( \bar{S}_{A,c},\, W_{A,c} \right) \right\}_{c=1}^{m_A},
\end{aligned}
\label{eq:merged_digest_clusters}
\end{equation}
with $ m_A = \Theta(\delta) $ clusters, where $ \mathrm{Merge}(\cdot) $ denotes the merge operation on $\{\mathsf{TD}_{k,A}\}_{k=1}^K$. 
As a result, the union over the digests of all atoms yields the final coreset:
\begin{equation}
\widetilde{\mathcal{D}} = \left\{ \left( A, \bar{S}, W \right) :\left( \bar{S},\, W \right)\in \mathsf{TD}_{A},\  A \in \mathcal{A} \right\}
\label{eq:coreset_def}
\end{equation}
with $|\widetilde{\mathcal{D}}|=\Theta(\delta|\mathcal{A}^+|)$.
Note that each entry $( A, \bar{S}, W )$ of the coreset includes the identifier of the atom $A$, the mean score $\bar{S}$, and the corresponding weight $W$.
{\color{black}
The uncompressed federated implementation would transmit all atom-stratified scores, and solve the same objective as centralized GC-FCP, thus requiring $\mathcal O(n)$ score records plus atom identifiers.
GC-FCP replaces the full dataset with the atom-stratified mergeable coreset \eqref{eq:coreset_def}.
}

\textbf{Set construction:}
Given a test score $ S $, GC-FCP solves problem \eqref{eq:gcfcp_augqr} using the coreset $\widetilde{\mathcal{D}}$ instead of the original pooled data $\mathcal{D}$, i.e.,
\begin{multline}
\tilde{\beta}(S) = \underset{\beta \in \mathbb{R}^{|\mathcal{G}|}}{\operatorname*{argmin}} \sum_{(A, \bar{S}, W) \in \widetilde{\mathcal{D}}} W \, \ell_\alpha\left( \beta^\top \Phi_{A},\, \bar{S} \right) \\
+ \left( \sum_{k=1}^K \frac{\pi_k}{n_k + 1} \right) \ell_\alpha\left( \beta^\top \Phi(X_{n+1}),\, S \right).
\label{eq:coreset_beta_augqr}
\end{multline}

Finally, the prediction set is constructed as
\begin{multline}
\mathcal{C}(X_{n+1} \mid \widetilde{\mathcal{D}}) \\ = \left\{ y \in \mathcal{Y} : s(X_{n+1}, y) \le \tilde{g}_{s(X_{n+1}, y)}(X_{n+1}) \right\}.
\label{eq:gcfcp_set_coreset}
\end{multline}
where $ \tilde{g}_S(x) =  \tilde{\beta}(S)^\top \Phi(x) $.

The proposed GC-FCP is summarized in Algorithm~\ref{alg:gcfcp}.

\textbf{Communication overhead:}
Solving problem~\eqref{eq:coreset_beta_augqr} at the server requires collecting active atom digests, yielding total communication load of order $\mathcal{O}(\delta\sum_{k=1}^{K}|\mathcal A_k^+|)$, compared to order $\mathcal{O}(n)$ for score-based communication of CondCP or centralized GC-FCP.

\textbf{Computational overhead:}
On the server side, solving the dual of problem~\eqref{eq:coreset_beta_augqr} requires computational complexity of order $\mathcal{O}(|\widetilde{\mathcal{D}}|^{3/2}|\mathcal{G}|^2) = \mathcal{O}\!\left((\delta|\mathcal A^+|)^{3/2}|\mathcal{G}|^2\right)$ by standard convex problem solvers~\citep{renegar1988newton,nesterov1994interior}, compared to the order $\mathcal{O}(n^{3/2}|\mathcal{G}|^2)$ of CondCP or centralized GC-FCP.
On client $k \in [K]$, membership computation costs $\mathcal O(n_k|\mathcal G|)$. 

\textbf{Memory overhead:}
As the digest memory scales as $\mathcal O(\delta|\mathcal A_k^+|)$, the server needs to store $\mathcal O(\delta|\mathcal A^+|)$ merged centroids.

\textbf{Trust and Privacy:}
GC-FCP keeps raw calibration examples local, but does not provide any formal privacy guarantee.
In particular, active membership patterns, atom-based counts, centroid weights, and approximate score distributions etc. may be exposed by T-Digests.
The present analysis therefore assumes a trustworthy server.
Privacy-preserving extensions possibly incorporate secure aggregation or differentially private quantile sketches~\citep{dwork2014algorithmic,alabi2022bounded} with potential compromise to coverage, which is left for future work.

\begin{algorithm}[t]
\caption{GC-FCP}
\label{alg:gcfcp}
\begin{algorithmic}
\STATE \textbf{Input:} Clients' calibration sets $\{\mathcal{D}_k\}_{k=1}^K$; groups $\mathcal{G}$; score $s(\cdot,\cdot)$; mis-coverage level $\alpha$; mixture weights $\pi_k$; T-Digest compression parameter $\delta$.
\STATE $\triangleright$ \texttt{Client side (in parallel): }
\FOR{device $k \in \{1,\ldots,K\}$}
\STATE \rev{Construct T-Digest $\mathsf{TD}_{k,A}$ for each active atom $A \in \mathcal{A}_k^+$ using $\mathcal{D}_{k,A}$.}
\STATE \rev{Transmit T-Digests $\{\mathsf{TD}_{k,A}\}_{A\in \mathcal{A}_k^+}$ to the central server.}
\ENDFOR
\STATE $\triangleright$ \texttt{Server side: }
\STATE \rev{Merge received digests for each active atom $A \in \mathcal{A}^+$ via \eqref{eq:merged_digest_clusters}.}
\STATE Solve \eqref{eq:coreset_beta_augqr} and compute the prediction set $\eqref{eq:gcfcp_set_coreset}$.
\STATE \textbf{Output:} $\mathcal{C}(\cdot\mid\widetilde{\mathcal{D}})$
\end{algorithmic}
\end{algorithm}

\section{Coverage Guarantees of GC-FCP}
\label{sec:theoretical_guarantee}
The key property of GC-FCP is its ability to provide group-conditional coverage guarantees \eqref{eq:group_conditional_coverage} for the prediction set. 
Deriving finite-sample group-conditional guarantees for GC-FCP is technically challenging, since one must simultaneously account for the non-exchangeability of the samples in the coreset and for the approximation error introduced by sketching via the local digest.
To this end, we first analyze the quantile estimation accuracy of T-Digest, and then provide a group-conditional coverage bound for GC-FCP.

\subsection{Quantile Accuracy of T-Digest}
We start by deriving a relationship between the compression parameter $\delta$ used in the scale function \eqref{eq:td_scale} by T-Digest for compression, which dictates the number of clusters $m=\Theta(\delta)$, and the accuracy of the approximate quantile \eqref{eq:tdigest_approximate_quantile}.

\begin{lemma}[Uniform bound of T-Digest]\label{lemma:tdigest_accuracy}
Defining as $F(t)$ the true CDF of original samples, the CDF estimate \eqref{eq:td_approximate_cdf} produced by T-Digest satisfies the uniform accuracy bound
\begin{equation}
\sup_{t\in\mathbb R}\bigl|F(t)-\widehat F(t)\bigr|
\le 
\sin\Bigl(\frac{\pi}{\delta}\Bigr),
\end{equation}
with $\delta \ge 2$.
Moreover, 
for all $u\in[0,1]$, the quantile estimate \eqref{eq:td_approximate_cdf} satisfies the inequality
$
| F(\widehat Q(u)) - u | \le \sin(\Myfrac{\pi}{\delta}).
$
\end{lemma}
\begin{proof}
    See Appendix~\ref{app:tdigest_accuracy}.
\end{proof}


\subsection{Group-Conditional Coverage Guarantees of GC-FCP}

Based on properties of T-Digest presented in Lemma~\ref{lemma:tdigest_accuracy}, we can now prove the group-conditional coverage guarantees of GC-FCP.

\begin{theorem}[Group-conditional coverage guarantees for GC-FCP] \label{thm:lower_bound}
For each group $ G \in \mathcal{G} $, the prediction set $\mathcal{C}(X_{n+1} \mid \widetilde{\mathcal{D}})$ produced by GC-FCP satisfies the inequality
\begin{multline} \label{eq:lower_bound_gcfcp}
\mathbb{P}(Y_{n+1} \in \mathcal{C}(X_{n+1} \mid \widetilde{\mathcal{D}}) \mid X_{n+1} \in G) \\
\geq 1 - \alpha - \sin\!\left(\frac{\pi}{\delta}\right).
\end{multline}
\end{theorem}
\textit{Proof:} See Appendix \ref{sec:proof-thm-lower-bound}.

While Theorem~\ref{thm:lower_bound} assumes that problems \eqref{eq:coreset_beta_augqr} are solved exactly, the next result accounts for suboptimality induced by numerically solving problem~\eqref{eq:coreset_beta_augqr}.

An upper bound of the group-conditional coverage of GC-FCP can be found in Appendix~\ref{appendix:upper_bound}.




\section{Experiments}
\label{sec:experiments}

We evaluate GC-FCP on (\emph{i}) a synthetic regression benchmark \citep{romano2019cqr}; (\emph{ii}) CIFAR-10 image classification \citep{krizhevsky2009learning}; and (\emph{iii}) PathMNIST medical image classification \citep{medmnistv2}.
Across all experiments, we compare the proposed scheme, GC-FCP, to the following benchmarks:
(\emph{i}) \textbf{vanilla centralized CP}, which provides the marginal guarantee \eqref{eq:marginal_coverage};
(\emph{ii}) \textbf{FedCP} \citep{lu2023federated}, which also satisfies \eqref{eq:marginal_coverage};
(\emph{iii}) \textbf{centralized CondCP} \citep{gibbs2023conformal};
and (\emph{iv}) \textbf{centralized GC-FCP}, which is described in Section~\ref{sec:centralized_gcfcp}.
CondCP and centralized GC-FCP are only evaluated on the small-scale synthetic dataset and CIFAR-10 due to their prohibitive computational complexity (see Table~\ref{tab:complexity}).
\rev{Appendix~\ref{appendix:imagenet_experiment} further reports a large-scale ImageNet-1K experiment with federated and group-wise baselines, and additional ablation experiments over group settings and compression.}

For each group $G\in\mathcal{G}$, we estimate group-conditional coverage on a test set $\mathcal{T}$ as
$
\widehat{\mathrm{cov}}(G)
=
\Myfrac{1}{|\mathcal{T}_G|}\sum_{(x,y)\in \mathcal{T}_G}\mathbf{1}\left\{y\in \mathcal{C}(x \mid \mathcal{D})\right\},
$
where $\mathcal{T}_G=\{(x,y)\in\mathcal{T}: x\in G\}$,
and report the empirical coverage $\widehat{\mathrm{cov}}(G)$ for all groups $G\in\mathcal{G}$.
For classification, we also report the average prediction set size
$({1}/{|\mathcal{T}|})\sum_{(x,y)\in\mathcal{T}}|\mathcal{C}(x \mid \mathcal{D})|$.
Computational complexity is evaluated by the average wall-clock time required for each method to construct a prediction set on the same platform.

\subsection{Synthetic Regression}
\label{sec:exp_synth}

\begin{figure*}[t] 
\centering
    \includegraphics[width=1\linewidth]{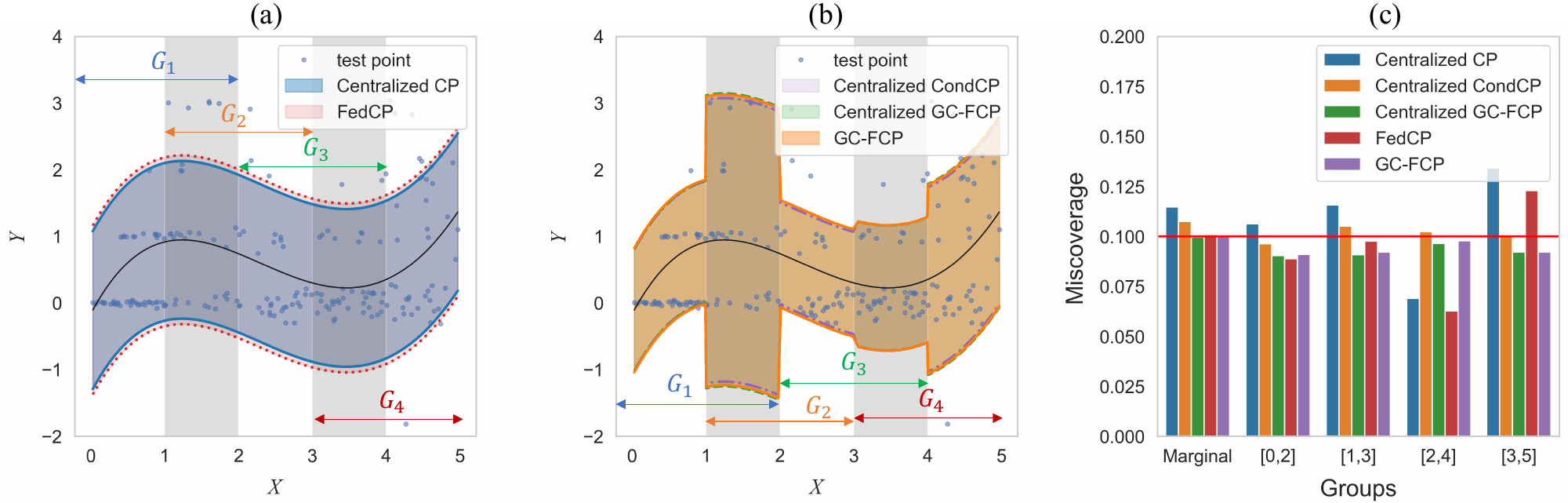}
    \caption{Visualization of prediction sets for the synthetic regression task for (a) centralized CP and FedCP~\citep{lu2023federated}, as well as for (b) centralized CondCP, centralized GC-FCP, and GC-FCP. (c) Per-group miscoverage rate.
}
\label{fig:synth_main}
\end{figure*}

For the synthetic regression task, we consider $K=4$ clients with heterogeneous covariate distributions $P_{X,k}$ given by truncated normal distributions on the interval $[0,5]$ with mean $\mu_k = 0.5 + \frac{4(k-1)}{K-1}$ and variance $\sigma_k = 0.5 + 0.1(k-1)$.

Following \citep{romano2020classification}, we generate responses as
\begin{multline} \label{eq:synth_y_gen}
Y_k \sim \mathrm{Pois}(\sin^2(X)+0.1) + 0.03X\epsilon_1 \\ + 25\mathbf{1}\{U<0.01\}\epsilon_2 + \mathcal{N}(0,0.01k^2),
\end{multline}
with independent variables $\epsilon_1,\epsilon_2\sim\mathcal{N}(0,1)$ and $U\sim\mathrm{Unif}([0,1])$. We train a centralized linear regression model $f(\cdot)$ on a separate training dataset generated in the same way, and use the score $s(x,y)=|y-f(x)|$. 
Groups are given by overlapping intervals $\mathcal{G}=\{[0,2],[1,3],[2,4],[3,5]\}$. 
The miscoverage level $\alpha$ is set to $\alpha=0.1$.
Other parameters are summarized in Table~\ref{tab:synth_setup} in Appendix~\ref{appendix:additional_results_of_main}.

Fig.~\ref{fig:synth_main} (a)-(b) visualizes representative prediction sets, while Fig.~\ref{fig:synth_main} reports the mis-coverage rate across the four overlapping interval groups.
Centralized CP and FedCP apply a single global threshold and therefore exhibit uneven group-wise miscoverage rates. 
In particular, some groups are under-covered (mis-coverage above $\alpha$), reflecting shifts in the conditional score distribution across covariate regions. In contrast, CondCP, centralized GC-FCP, and GC-FCP yield substantially more uniform group-wise mis-coverage near the target level $\alpha$ across all groups (Fig.~\ref{fig:synth_main}(c)) by adapting the threshold over the covariate space (Fig.~\ref{fig:synth_main}(b)).

\subsection{CIFAR-10 Experiments}
\label{sec:exp_cifar}

%


\begin{table*}[t]
\centering
\caption{Coverage, set size, and computational load comparisons on CIFAR-10~\citep{krizhevsky2009learning}. Each entry reports mean $\pm$ standard error over Monte Carlo repetitions.  Standard errors for set size below $0.01$ are omitted. A red star ${}^{{\color{red}\star}}$ indicates coverage below the target $1-\alpha=0.9$.}
{\color{black}
\begin{tabular}{ccccccc}
\hline
\multirow{2}{*}{\textbf{Methods}} & \multirow{2}{*}{\textbf{\begin{tabular}[c]{@{}c@{}}Marginal\\ coverage\end{tabular}}} & \multicolumn{4}{c}{\textbf{Group-conditional coverage (set size)}} & \multirow{2}{*}{\textbf{\begin{tabular}[c]{@{}c@{}}Comp.\\ speedup\end{tabular}}} \\ \cline{3-6}
 &  & \textbf{$G_1$} & \textbf{$G_2$} & \textbf{$G_3$} & \textbf{$G_4$} &  \\ \hline
\begin{tabular}[c]{@{}c@{}} Centralized \\ CP \end{tabular} & $0.901\pm0.001$ & \begin{tabular}[c]{@{}c@{}}$0.879^{{\color{red}\star}}\pm0.001$ \\ ($0.91$)\end{tabular} & \begin{tabular}[c]{@{}c@{}}$0.855^{{\color{red}\star}}\pm0.001$\\ ($0.90$)\end{tabular} & \begin{tabular}[c]{@{}c@{}}$0.906\pm0.001$\\ ($0.93$)\end{tabular} & \begin{tabular}[c]{@{}c@{}}$0.936\pm0.001$\\ ($0.95$)\end{tabular} & N/A \\ \hline
\begin{tabular}[c]{@{}c@{}} Centralized \\ CondCP \end{tabular} & $0.903\pm0.001$ & \begin{tabular}[c]{@{}c@{}}$0.901\pm0.001$\\ ($0.96$)\end{tabular} & \begin{tabular}[c]{@{}c@{}}$0.901\pm0.001$\\ ($0.98$)\end{tabular} & \begin{tabular}[c]{@{}c@{}}$0.902\pm0.001$\\ ($0.94$)\end{tabular} & \begin{tabular}[c]{@{}c@{}}$0.901\pm0.001$\\ ($0.91$)\end{tabular} & $1.00\times$ \\ \hline
\begin{tabular}[c]{@{}c@{}} Centralized \\ GC-FCP \end{tabular} & $0.906\pm0.001$ & \begin{tabular}[c]{@{}c@{}}$0.903\pm0.001$\\ ($0.96$)\end{tabular} & \begin{tabular}[c]{@{}c@{}}$0.903\pm0.001$\\ ($0.99$)\end{tabular} & \begin{tabular}[c]{@{}c@{}}$0.905\pm0.001$\\ ($0.94$)\end{tabular} & \begin{tabular}[c]{@{}c@{}}$0.905\pm0.001$\\ ($0.92$)\end{tabular} & $1.00\times$ \\ \hline
FedCP & $0.902\pm0.001$ & \begin{tabular}[c]{@{}c@{}}$0.880^{{\color{red}\star}}\pm0.001$ \\ ($0.92$)\end{tabular} & \begin{tabular}[c]{@{}c@{}}$0.856^{{\color{red}\star}}\pm0.001$ \\ ($0.90$)\end{tabular} & \begin{tabular}[c]{@{}c@{}}$0.906\pm0.001$\\ ($0.93$)\end{tabular} & \begin{tabular}[c]{@{}c@{}}$0.936\pm0.001$\\ ($0.95$)\end{tabular} & N/A \\ \hline
\begin{tabular}[c]{@{}c@{}} GC-FCP \\ ($\delta=25$) \end{tabular} & $0.911\pm0.001$ & \begin{tabular}[c]{@{}c@{}}$0.896^{{\color{red}\star}}\pm0.002$ \\ ($0.95$)\end{tabular} & \begin{tabular}[c]{@{}c@{}}$0.887^{{\color{red}\star}}\pm0.002$ \\ ($0.96$)\end{tabular} & \begin{tabular}[c]{@{}c@{}}$0.919\pm0.002$\\ ($0.95$)\end{tabular} & \begin{tabular}[c]{@{}c@{}}$0.931\pm0.002$\\ ($0.95$)\end{tabular} & $\mathbf{34.36\times}$ \\ \hline
\begin{tabular}[c]{@{}c@{}} GC-FCP \\ ($\delta=250$) \end{tabular} & $0.906\pm0.001$ & \begin{tabular}[c]{@{}c@{}}$0.903\pm0.001$\\ ($0.96$)\end{tabular} & \begin{tabular}[c]{@{}c@{}}$0.902\pm0.001$\\ ($0.98$)\end{tabular} & \begin{tabular}[c]{@{}c@{}}$0.905\pm0.001$\\ ($0.94$)\end{tabular} & \begin{tabular}[c]{@{}c@{}}$0.906\pm0.001$\\ ($0.92$)\end{tabular} & $\mathbf{19.73\times}$ \\ \hline
\begin{tabular}[c]{@{}c@{}} GC-FCP \\ ($\delta=2500$) \end{tabular} & $0.906\pm0.001$ & \begin{tabular}[c]{@{}c@{}}$0.903\pm0.001$\\ ($0.96$)\end{tabular} & \begin{tabular}[c]{@{}c@{}}$0.903\pm0.001$\\ ($0.99$)\end{tabular} & \begin{tabular}[c]{@{}c@{}}$0.905\pm0.001$\\ ($0.94$)\end{tabular} & \begin{tabular}[c]{@{}c@{}}$0.905\pm0.001$\\ ($0.92$)\end{tabular} & $\mathbf{3.96\times}$ \\ \hline
\end{tabular}
}
\label{tab:experimental_result_cifar10}
\end{table*}

\begin{table*}[ht]
\centering
\caption{Coverage and set size comparisons on PathMNIST. Each entry reports mean $\pm$ standard error over Monte Carlo repetitions. Standard errors of set size below $0.01$ are omitted. A red star ${}^{{\color{red}\star}}$ indicates the coverage below the target $1-\alpha=0.9$.}
{\color{black}
\begin{tabular}{cccccc}
\hline
\multirow{2}{*}{\textbf{\begin{tabular}[c]{@{}c@{}}Methods \\ (marginal coverage)\end{tabular}}} & \multicolumn{5}{c}{\textbf{Group-conditional coverage (set size)}} \\ \cline{2-6} 
 & \textbf{$G_1$} & \textbf{$G_2$} & \textbf{$G_3$} & \textbf{$G_4$} & \textbf{$G_5$} \\ \hline
\begin{tabular}[c]{@{}c@{}} Centralized CP \\ ($0.901\pm0.001$) \end{tabular} & \begin{tabular}[c]{@{}c@{}}$0.888^{{\color{red}\star}} \pm0.001$\\ ($1.00$)\end{tabular} & \begin{tabular}[c]{@{}c@{}}$0.885^{{\color{red}\star}} \pm0.001$\\ ($1.00$)\end{tabular} & \begin{tabular}[c]{@{}c@{}}$0.879^{{\color{red}\star}} \pm0.001$\\ ($1.00$)\end{tabular} & \begin{tabular}[c]{@{}c@{}}$0.935 \pm0.001$\\ ($1.00$)\end{tabular} & \begin{tabular}[c]{@{}c@{}}$0.895^{{\color{red}\star}} \pm0.001$\\ ($1.00$)\end{tabular} \\ \hline
\begin{tabular}[c]{@{}c@{}} FedCP \\ ($0.905\pm0.001$) \end{tabular} & \begin{tabular}[c]{@{}c@{}}$0.892^{{\color{red}\star}} \pm0.001$\\ ($1.01$)\end{tabular} & \begin{tabular}[c]{@{}c@{}}$0.889^{{\color{red}\star}} \pm0.001$\\ ($1.01$)\end{tabular} & \begin{tabular}[c]{@{}c@{}}$0.883^{{\color{red}\star}} \pm0.001$\\ ($1.02$)\end{tabular} & \begin{tabular}[c]{@{}c@{}}$0.937 \pm0.001$\\ ($1.01$)\end{tabular} & \begin{tabular}[c]{@{}c@{}}$0.900 \pm0.001$\\ ($1.02$)\end{tabular} \\ \hline
\begin{tabular}[c]{@{}c@{}} GC-FCP ($\delta=25$) \\ ($0.876^{{\color{red}\star}}\pm0.002$) \end{tabular} & \begin{tabular}[c]{@{}c@{}}$0.891^{{\color{red}\star}} \pm0.002$\\ ($1.12$)\end{tabular} & \begin{tabular}[c]{@{}c@{}}$0.890^{{\color{red}\star}} \pm0.003$\\ ($1.14\pm0.01$)\end{tabular} & \begin{tabular}[c]{@{}c@{}}$0.880^{{\color{red}\star}} \pm0.003$\\ ($1.15\pm0.09$)\end{tabular} & \begin{tabular}[c]{@{}c@{}}$0.895^{{\color{red}\star}} \pm0.004$\\ ($1.07\pm0.09$)\end{tabular} & \begin{tabular}[c]{@{}c@{}}$0.853^{{\color{red}\star}} \pm0.003$\\ ($1.01\pm0.05$)\end{tabular} \\ \hline
\begin{tabular}[c]{@{}c@{}} GC-FCP ($\delta=250$) \\ ($0.908\pm0.001$) \end{tabular} & \begin{tabular}[c]{@{}c@{}}$0.913 \pm0.001$\\ ($1.07$)\end{tabular} & \begin{tabular}[c]{@{}c@{}}$0.910 \pm0.001$\\ ($1.08$)\end{tabular} & \begin{tabular}[c]{@{}c@{}}$0.911 \pm0.001$\\ ($1.11$)\end{tabular} & \begin{tabular}[c]{@{}c@{}}$0.911 \pm0.001$\\ ($0.98$)\end{tabular} & \begin{tabular}[c]{@{}c@{}}$0.898 \pm0.001$\\ ($1.03$)\end{tabular} \\ \hline
\begin{tabular}[c]{@{}c@{}} GC-FCP ($\delta=2500$) \\ ($0.909\pm0.001$) \end{tabular} & \begin{tabular}[c]{@{}c@{}}$0.913 \pm0.001$\\ ($1.07$)\end{tabular} & \begin{tabular}[c]{@{}c@{}}$0.910 \pm0.001$\\ ($1.08$)\end{tabular} & \begin{tabular}[c]{@{}c@{}}$0.910 \pm0.001$\\ ($1.11$)\end{tabular} & \begin{tabular}[c]{@{}c@{}}$0.912 \pm0.001$\\ ($0.98$)\end{tabular} & \begin{tabular}[c]{@{}c@{}}$0.900 \pm0.001$\\ ($1.03$)\end{tabular} \\ \hline
\end{tabular}
}
\label{tab:experimental_result_pathmnist}
\end{table*}

We now adopt a pre-trained ResNet56 model $f(\cdot)$ for the CIFAR-10 image classification task with score function $s(x,y)=1-[f(x)]_y$, where $y$ denotes the $y$-th entry of the softmax result \citep{sadinle2019least}. We consider $K=5$ clients with a non-i.i.d.\ label partition, so that each client holds $10/K$ disjoint classes. We set uniform mixture weights $\pi_k=1/K$, and total calibration size $n=\sum_k n_k=5000$. Groups are defined by overlapping predicted label classes $\hat y(x)=\arg\max_y [f(x)]_y$:
\(
G_1=\{x:\hat y(x)\in\{0,1,2,3\}\},
G_2=\{x:\hat y(x)\in\{2,3,4,5\}\},
G_3=\{x:\hat y(x)\in\{4,5,6,7\}\}, \text{and }
G_4=\{x:\hat y(x)\in\{6,7,8,9\}\}.
\)
We set the miscoverage level $\alpha=0.1$ and use a pre-trained ResNet56 to serve the model $f(\cdot)$ to construct the score function $s(x,y)=1-[f(x)]_y$~\citep{sadinle2019least}. The report results represent the average after $50$ Monte Carlo simulations, each with $5000$ test points.

Table~\ref{tab:experimental_result_cifar10} reports marginal coverage, group-conditional coverage (with average set size in brackets), and computational speedup (``Comp. speedup''), with the latter being normalized with respect to the complexity of centralized CondCP.
Centralized CP and FedCP attain the same marginal coverage near $0.9$, but their group-conditional coverage varies substantially across groups (e.g., below $0.9$ for groups $G_1$ and $G_2$), indicating that marginal calibration does not control errors uniformly across overlapping groups. Centralized CondCP and centralized GC-FCP achieve near-uniform group-conditional coverage across all groups, matching the intended behavior of augmented quantile calibration under grouping conditions.

GC-FCP closely tracks the centralized group-conditional performance while substantially improving computational efficiency. 
The computational gains are pronounced: depending on $\delta$, GC-FCP yields from $3.96\times$ to $34.36\times$ per-test-point speedup relative to centralized CondCP, illustrating the accuracy--efficiency trade-off governed by the digest compression parameter.
The value of the compression parameter affects coverage. For instance, with aggressive compression ($\delta=25$), some groups are slightly under-covered (e.g., $G_2$), in a manner consistent with Theorem~\ref{thm:lower_bound}.

\subsection{Medical dataset experiments}
\label{sec:exp_med}


Finally, we evaluate GC-FCP on PathMNIST from MedMNIST (9 tissue classes) with images resized to $3\times 28\times 28$ \citep{medmnistv2}.
We consider $K=5$ clients, each holding samples from $2$ disjoint classes, except the last client holds $1$ class.
We train a centralized CNN $f(\cdot)$ on the training split.
We randomly shuffle the mixture validation and test data, split it into calibration and conformal test datasets equally,
and allocate calibration samples to clients as in the CIFAR-10 experiments.
Groups are defined by predicted-label classes:
$
G_1 = \{x : \hat{y}(x) \in \{0,1,2\}\}, 
G_2 = \{x : \hat{y}(x) \in \{1,2,3\}\}, 
G_3 = \{x : \hat{y}(x) \in \{2,3,4\}\}, 
G_4 = \{x : \hat{y}(x) \in \{3,4,5\}\},
G_5 = \{x : \hat{y}(x) \in \{4,5,6,7,8\}\}.
$
Other parameters remain the same as CIFAR-10 experiments.

Table~\ref{tab:experimental_result_pathmnist}
shows that Centralized CP and FedCP achieve marginal coverage near $0.9$, but can deviate noticeably at the group level (e.g., lower coverage on $G_3$). 
GC-FCP with moderate compression ($\delta=250$ or $\delta=2500$) attains group-conditional coverage close to the target for all groups, while maintaining small average set sizes (near one label on average). 
In contrast, overly aggressive compression ($\delta=25$) degrades calibration. 
Overall, these results corroborate that GC-FCP achieves group-conditional reliability in heterogeneous federated classification tasks, with accuracy controlled by the digest compression parameter $\delta$.



\section{Conclusion}
\label{sec:conclusion}

We have introduced \emph{Group-Conditional Federated Conformal Prediction (GC-FCP)}, a principled framework for constructing prediction sets with group-conditional guarantee in heterogeneous federated settings. 
GC-FCP leverages T-Digest to compress atom-stratified calibration scores into a mergeable coreset, achieving communication and computation efficiency when the number of active atoms is moderate.
Under mild assumptions, we established group-conditional coverage guarantees for GC-FCP, and empirical results on synthetic regression and image classification benchmarks corroborated these findings while demonstrating substantial computational speedups. 

GC-FCP assumes a prespecified finite group family and a trustworthy server for atom-based summaries.
Future work includes extending the framework to richer conditional structures beyond finite group families, developing privacy-preserving sketches, as well as investigating robustness against more complex settings such as decentralized calibration~\citep{wen2025distributed}, adversarial behavior~\citep{kang2024certifiably}, and online CP~\citep{angelopoulos2024online,gasparin2024conformal}.





\bibliography{uai2026-template}

\newpage

\onecolumn

\title{Efficient Federated Conformal Prediction with Group-Conditional Guarantee \\ (Supplementary Material)}
\maketitle


\appendix

\section{Proof of Section~\ref{sec:method}}
\subsection{Proof of Theorem~\ref{thm:gcfcp_exact_group}}
\label{sec:proof-thm-exact-group}
\begin{proof}[Proof]
First assume there are no ties, i.e.,
$S_{i,k} \neq \hat g_{S_{n+1}}(X_{i,k})$ for all $i\in[n_k],\,k\in[K]$, and
$S_{n+1} \neq \hat g_{S_{n+1}}(X_{n+1})$.
The tie case is handled at the end of the proof.
Recall that $g(x)=\beta^\top \Phi(x)$ with $\beta\in\mathbb{R}^{|\mathcal{G}|}$ and $\Phi(\cdot)$ defined in~\eqref{eq:membership_vector}, 
then the optimal function $\hat{g}_{S_{n+1}}(\cdot)$ solved by $\eqref{eq:gcfcp_augqr}$ reduces to a real-value vector given by
\begin{equation}
    \beta^* = \underset{\beta \in\mathbb{R}^{|\mathcal{G}|}}{\operatorname*{\operatorname*{argmin}}}
\sum_{k=1}^{K} \bigg (\frac{\pi_k }{n_k+1} 
\sum_{i=1}^{n_k+1}\ell_\alpha\bigg(\sum_{G\in\mathcal{G}}\beta_G\mathbf{1}\{X_{i,k}\in G\},S_{i,k}\bigg) \bigg),
\end{equation}
where we denote $S_{n_k+1,k}=S_{n+1}$ for simplicity. 
Under the no-ties assumption, the first-order optimality condition implies that for every group $G\in\mathcal{G}$,
\begin{equation}
\sum_{k=1}^K\frac{\pi_k}{n_k+1}\sum_{i=1}^{n_k+1}\Big(\mathbf{1}\{X_{i,k}\in G\} \left(\mathbf{1}\{S_{i,k}<\theta_{i,k}^*\}-(1-\alpha)\right) \Big)=0, \quad \forall G\in \mathcal{G}
\end{equation}
where $\theta_{i,k}^*=\sum_{G\in\mathcal{G}}\beta^*_G\mathbf{1}\{X_{i,k}\in G\}$. 

Let $E_k$ be the event that $(X_{n+1},Y_{n+1})$ is drawn from $\mathrm{P}_k$.
Conditioned on $E_k$, the $(n_k+1)$ scores $\{S_{i,k}\}_{i=1}^{n_k+1}$ are exchangeable.
Let $\mathcal{E}$ be the event that, for every $k\in[K]$, there exists a permutation $\sigma_k$ such that
\begin{equation}
\big(S_{\sigma_k(1),k},\ldots,S_{\sigma_k(n_k+1),k}\big)
=
\big(s_{1,k},\ldots,s_{n_k+1,k}\big),
\tag{40}
\end{equation}
where $(s_{1,k},\ldots,s_{n_k+1,k})$ denotes the realized values of $(S_{1,k},\ldots,S_{n_k+1,k})$.
Then, we have, for all $G \in \mathcal{G}$,
\begin{equation}
\begin{aligned}
& \mathbb{P}(Y_{n+1} \in \mathcal{C}(X_{n+1}\mid\mathcal{D}) \mid X_{n+1}\in G,\mathcal{E}) \\ 
& = \frac{\mathbb{P}(S_{n+1}\le \theta^*_{n+1}, X_{n+1}\in G \mid \mathcal{E})}{\mathbb{P}(X_{n+1}\in G \mid \mathcal{E})} \\ 
& = \frac{\sum_{k=1}^{K}\pi_k \mathbb{P}(S_{n+1}\le \theta^*_{n+1}, X_{n+1}\in G \mid \mathcal{E}, \mathcal{E}_k)}{\sum_{k=1}^{K}\pi_k \mathbb{P}(X_{n+1}\in G \mid \mathcal{E}, \mathcal{E}_k)} \\ 
& \overset{(a)}{=} \frac{\sum_{k=1}^{K}\frac{\pi_k}{n_k+1} \sum_{i=1}^{n_k+1}\mathbf{1}\{S_{i,k}\le \theta^*_{i,k}\} \mathbf{1}\{X_{i,k}\in G\} }{\sum_{k=1}^{K} \frac{\pi_k}{n_k+1} \sum_{i=1}^{n_k+1}\mathbf{1}\{X_{i,k}\in G\} }  \\
& \overset{(b)}{=} \frac{(1-\alpha)\sum_{k=1}^{K} \frac{\pi_k}{n_k+1} \sum_{i=1}^{n_k+1}\mathbf{1}\{X_{i,k}\in G\} }{\sum_{k=1}^{K} \frac{\pi_k}{n_k+1} \sum_{i=1}^{n_k+1}\mathbf{1}\{X_{i,k}\in G\} } \\ 
& = 1-\alpha,
\end{aligned}
\end{equation}
where $(a)$ follows exchangeability on client $k$ under the event $E_k$ and $(b)$ follows the first-order condition. 

This result implies marginal coverage under $G = \mathcal{X}$. If the assumption $S_i\ne \hat{g}_S(X_i)$ does not hold, we have 
$$
\mathbb{P}(Y_{n+1} \in \mathcal{C}(X_{n+1}\mid\mathcal{D}) \mid X_{n+1}\in G,\mathcal{E}) \ge 1-\alpha.
$$
The proof technique is similar to \citep{gibbs2023conformal} and Appendix~\ref{sec:proof-thm-lower-bound} and is omitted here. By the tower rule, we obtain the desired result.
\end{proof}

\section{Proofs for Section~\ref{sec:theoretical_guarantee}}
\label{sec:appendix-proofs}

\subsection{Results on T-Digest}\label{app:tdigest_accuracy}

This section provides additional results on T-Digest and proves Lemma~\ref{lemma:tdigest_accuracy}. 

To begin, we state the following fact that holds in our digest construction described in Section~\ref{sec:gc-fcp}.

\subsubsection{Important Lemmas}
\begin{lemma}[CDF error controlled by maximal cluster mass]\label{lem:tdigest_cdf_mass}
Let $\{(R_i,w_i)\}_{i=1}^\ell$ be weighted samples with total weight $W$ and empirical CDF $F$.
Let $\mathsf{TD}=\{(\bar R_c,W_c)\}_{c=1}^{m}$ be a digest with induced CDF $\widehat F(t)\ =\ \frac{1}{W}\sum_{c=1}^{m} W_c\,\mathbf 1\{\bar R_c\le t\}$.
Define the maximal normalized cluster mass
\begin{equation}
\rho_{\max}\ =\ \max_{1\le c\le m }\frac{W_c}{W}.
\end{equation}
Then,
\begin{equation}
\sup_{t\in\mathbb R}\bigl|F(t)-\widehat F(t)\bigr|\le\rho_{\max}.
\end{equation}
\end{lemma}

\begin{proof}
Let $\{\mathcal{R}_c\}_{c=1}^{m}$ be the partition for the digest defined in Section~\ref{sec:gc-fcp}, and define
$L_c=\min_{i\in \mathcal{R}_c}R_i$ and $U_c=\max_{i\in \mathcal{R}_c}R_i$.
Fix $t\in\mathbb R$ and consider any cluster $c$. If $t<L_c$, then cluster $c$ contributes zero to both $F(t)$ and $\widehat F(t)$.
If $t\ge U_c$, cluster $c$ contributes $W_c/W$ to both $F(t)$ and $\widehat F(t)$.
If $t\in[L_c,U_c)$, discrepancy arises from the cluster $c$ only because $\max_{i\in \mathcal{R}_c} R_i \le\min_{i\in \mathcal{R}_{c^\prime}} R_i$ for any $c'>c$, which yields
\begin{equation}
\bigl|F(t)-\widehat F(t)\bigr|
\le
\frac{W_c}{W}.
\end{equation}
Taking the supremum over $t$ completes the proof.
\end{proof}

\begin{lemma}[Arcsine scale implies $\rho_{\max}\le \sin(\pi/\delta)$]\label{lem:tdigest_mass_arcsine}
Suppose the digest is constructed with the scale $r(q)=\frac{\delta}{2\pi}\arcsin(2q-1)$ under the condition of \eqref{eq:td_size_criterion} with $\delta\ge2$. Then
\begin{equation}
\rho_{\max}=\max_{1\le c\le m}\frac{W_c}{W}\le\sin\Bigl(\frac{\pi}{\delta}\Bigr) \le \frac{\pi}{\delta}.
\end{equation}
\end{lemma}

\begin{proof}
For each cluster $c$, let its empirical left/right quantile boundaries be
$q_c^L=V_{c-1}/W$ and $q_c^R=V_c/W$, where $V_c=\sum_{i=1}^c W_i$ (with $V_0=0$).
Then $q_c^R-q_c^L=W_c/W$. The well-formedness constraint gives
\begin{equation}
r(q_c^R)-r(q_c^L)\le1
\quad\Longleftrightarrow\quad
\arcsin(2q_c^R-1)-\arcsin(2q_c^L-1)\le\frac{2\pi}{\delta}.
\end{equation}
Let $u_L=2q_c^L-1$ and $u_R=2q_c^R-1$, so that $u_L,u_R\in[-1,1]$ and
$q_c^R-q_c^L=(u_R-u_L)/2$. Under the constraint
$\arcsin(u_R)-\arcsin(u_L)\le 2\pi/\delta$, the difference $u_R-u_L$ is maximized by taking symmetry around zero:
$\arcsin(u_R)=\pi/\delta$ and $\arcsin(u_L)=-\pi/\delta$, hence $u_R=\sin(\pi/\delta)$ and $u_L=-\sin(\pi/\delta)$.
Therefore
\begin{equation}
u_R-u_L\le2\sin\Bigl(\frac{\pi}{\delta}\Bigr)
\quad\Longrightarrow\quad
q_c^R-q_c^L=\frac{u_R-u_L}{2}\le\sin\Bigl(\frac{\pi}{\delta}\Bigr),
\end{equation}
which is exactly $W_c/W\le \sin(\pi/\delta)$. Taking the maximum over $c$ yields the claim.
\end{proof}

\begin{lemma}[From $\|F-\widehat F\|_\infty$ to rank-accurate quantiles]\label{lem:cdf_to_quantile}
Let $F,\widehat F$ be CDFs on $\mathbb R$ and define the generalized quantile function $\widehat Q(u)=\inf\{t:\widehat F(t)\ge u\}$.
If $\sup_{t\in\mathbb R}|F(t)-\widehat F(t)|\le \epsilon$, then for all $u\in[0,1]$,
\begin{equation}
F(\widehat Q(u))\ \in\ [u-\epsilon,\ u+\epsilon].
\end{equation}
In particular, when $F$ is an empirical CDF with total weight $W$, $\widehat Q(u)$ has empirical rank within
$(u\pm \epsilon)W$.
\end{lemma}

\begin{proof}
Fix $u\in[0,1]$ and let $t^\star=\widehat Q(u)$. By definition, $\widehat F(t^\star)\ge u$ and for any $t<t^\star$,
$\widehat F(t)<u$. Using $\sup_t|F(t)-\widehat F(t)|\le\epsilon$ gives
$F(t^\star)\ge \widehat F(t^\star)-\epsilon \ge u-\epsilon$.
For the upper bound, for any $\eta>0$ we have $\widehat F(t^\star-\eta)<u$, hence
$F(t^\star-\eta)\le \widehat F(t^\star-\eta)+\epsilon < u+\epsilon$.
Letting $\eta \to 0$ and using right-continuity of $F$ yields $F(t^\star)\le u+\epsilon$.
\end{proof}

\subsubsection{Proof of Lemma~\ref{lemma:tdigest_accuracy}}

\begin{proof}
By Lemma~\ref{lem:tdigest_cdf_mass},
\begin{equation}
\sup_t|F(t)-\widehat F(t)|\le\rho_{\max}.
\end{equation}
By Lemma~\ref{lem:tdigest_mass_arcsine}, $\rho_{\max}\le \sin(\pi/\delta)$, hence
$\sup_t|F(t)-\widehat F(t)|\le \sin(\pi/\delta)$.
\end{proof}

\subsection{Proof of Theorem~\ref{thm:lower_bound}}
\label{sec:proof-thm-lower-bound}

We begin with the following corollary that controls the error of T-Digest for each atom.
\begin{corollary}[Atom-wise sketch accuracy]\label{cor:atomwise_accuracy}
Let $p_A$ be its weighted empirical CDF and let $\widehat p_A$ be the step-CDF induced by the merged digest $\mathsf{TD}_A$. Then, under the arcsine scale \eqref{eq:td_scale},
\begin{equation}
\sup_{t\in\mathbb R}\bigl|p_A(t)-\widehat p_A(t)\bigr|
\le
\epsilon
=\sin\!\left(\frac{\pi}{\delta}\right)\le\frac{\pi}{\delta}.
\end{equation}
Equivalently, the quantile query induced by $\mathsf{TD}_A$ is $\epsilon$-accurate in the sense of Section~4.3.
\end{corollary}
\begin{proof}
This is a direct application of Lemma~\ref{lemma:tdigest_accuracy} to the weighted sample set $\mathcal D_A$.
\end{proof}

\paragraph{Step 1: First-order condition.} 
The subgradient of \eqref{eq:coreset_beta_augqr} on $\hat{\beta}$ with respect to $\beta$ is given by
\begin{multline}
    \left\{ \sum_{A \in \mathcal{A}} \sum_{c=1}^{m_A} W_{A,c} v_{A,c} \Phi_{A} + w_{\text{test}} v_{\text{test}} \Phi(X_{n+1}) \Bigg| \right. \\ 
    \left. v_{A,c} = \begin{cases} \alpha & \text{if } \bar{S}_{A,c} < \hat{\beta}^T \Phi_{A}, \\ -(1-\alpha) & \text{if } \bar{S}_{A,c} > \hat{\beta}^T \Phi_{A}, \\ t_{A,c} & \text{if } \bar{S}_{A,c} = \hat{\beta}^T \Phi_{A} \end{cases},  v_{\text{test}} = \begin{cases} \alpha & \text{if } S < \hat{\beta}^T \Phi(X_{n+1}), \\ -(1-\alpha) & \text{if } S > \hat{\beta}^T \Phi(X_{n+1}), \\ t_{\text{test}} & \text{if } S = \hat{\beta}^T \Phi(X_{n+1}) \end{cases} \right\},
\end{multline}

where $w_{\text{test}}= \sum_{k=1}^K \lambda_k$ with $\lambda_k=\pi_k/(n_k+1)$ and $t_{A,c}, t_{\text{test}} \in [\alpha-1, \alpha]$.

The first-order condition for optimality implies that for each group $G \in \mathcal{G}$, there exist $t_{A,c} \in [\alpha-1, \alpha]$ (for clusters with ties) and $t_{\text{test}} \in [\alpha-1, \alpha]$ (if the test point has a tie) such that
$$
\sum_{A: A \subseteq G} \sum_{c=1}^{m_A} W_{A,c} v_{A,c} + w_{\text{test}} v_{\text{test}} \mathbf{1}\{X_{n+1} \in G\} = 0,
$$
where the $v_{A,c}$ and $v_{\text{test}}$ are as defined above.

\paragraph{Step 2: Empirical CDF and T-Digest approximate.} 
The atoms $\{A\}$ partition $\mathcal{X}$ disjointly, with membership vectors $\Phi_{A}$. The true empirical CDF per atom is
$$
p_A(t) = W_A^{-1} \sum_{i: X_i \in A} \lambda_{k(i)} \mathbf{1}\{S_i \leq t\},
$$
where $W_A= \sum_{i:X_i \in A} \lambda_{k(i)}$, $\lambda_{k(i)} = \pi_{k(i)} / (n_{k(i)} + 1)$, and $k(i)$ is the index of the client that the sample $X_i$ belongs to, for $i\in[n]$. The coreset per atom yields
$$
\hat{p}_A(t) = W_A^{-1} \sum_{c = 1}^{m_A} W_{A,c} \mathbf{1}\{\bar{S}_{A,c} \leq t\},
$$
with $\sup_t |p_A(t) - \hat{p}_A(t)| \leq \epsilon$ by Corollary~\ref{cor:atomwise_accuracy}.

\paragraph{Step 3: Coverage.} 
With $(X_{n+1},Y_{n+1})\sim \sum_{k=1}^{K}\pi_k \mathrm{P}_k$, as in the proof of FCP \citep{lu2023federated}, let $E_k$ be the event that $(X_{n+1},Y_{n+1})$ is drawn from $\mathrm{P}_k$. 
For each $k\in[K]$, let $\Pi_{n_k+1}$ denote the set of all permutations of $\{1,\ldots,n_k+1\}$, and define the event
\begin{equation}
\mathcal{E}
=\left\{\forall k\in[K],\, \exists \sigma_k\in \Pi_{n_k+1}\ \text{s.t.}\ 
\big(S_{\sigma_k(1),k},\ldots,S_{\sigma_k(n_k+1),k}\big)
=\big(s_{1,k},\ldots,s_{n_k+1,k}\big)\right\},
\end{equation}
where $(s_{1,k},\ldots,s_{n_k+1,k})$ denotes the realized values of $(S_{1,k},\ldots,S_{n_k+1,k})$.

Then, we have, for all $G\in \mathcal{G}$,
$$
\mathbb{P}(Y_{n+1} \in \mathcal{C}(X_{n+1} \mid \widetilde{\mathcal{D}}) \mid X_{n+1} \in G, \mathcal{E}) = \mathbb{E}_{E_k}\left[\mathbb{P}(S_{n+1}\leq \hat{\beta}_{S_{n+1}}^T \Phi(X_{n+1}) \mid X_{n+1} \in G, \mathcal{E}, E_k) \right] = \frac{N}{D},
$$
where
$$
\begin{aligned}
N & = \sum_{A: A \subseteq G} W_Ap_A(\hat{\theta}_A) + w_{\text{test}} \mathbf{1}\{S_{n+1} \le \hat{\theta}_{n+1}\} \mathbf{1}\{X_{n+1} \in G\}, \\
D & = \sum_{A: A \subseteq G} W_A+ w_{\text{test}} \mathbf{1}\{X_{n+1} \in G\}
\end{aligned}
$$
with $\hat{\theta}_A = \hat{\beta}_{S_{n+1}}^T\Phi_{A}$ and $\hat{\theta}_{n+1}=\hat{\beta}_{S_{n+1}}^T \Phi(X_{n+1})$

Analogously, define the GC-FCP approximate numerator and denominator as
$$
\hat{N} = \sum_{A: A \subseteq G} W_A\hat{p}_A(\hat{\theta}_A) + w_{\text{test}} \mathbf{1}\{S_{n+1} \le \hat{\theta}_{n+1}\} \mathbf{1}\{X_{n+1} \in G\}, \quad \hat{D} = D.
$$
By Corollary~\ref{cor:atomwise_accuracy}, we have
$$
N \ge \sum_{A: A \subseteq G} W_A(\hat{p}_A(\hat{\theta}_A) - \epsilon) + w_{\text{test}} \mathbf{1}\{S_{n+1} \le \hat{\theta}_{n+1}\} \mathbf{1}\{X_{n+1} \in G\} = \hat{N} - \epsilon\sum_{A: A \subseteq G} W_A.
$$
By the first-order condition, for the component corresponding to $G$, there exist $t_{A,c}, t_{\text{test}} \in [\alpha - 1, \alpha]$ such that
\begin{multline}
\alpha \sum_{A: A \subseteq G} \sum_{c: \bar{S}_{A,c} < \hat{\theta}_A} W_{A,c} - (1 - \alpha) \sum_{A: A \subseteq G} \sum_{c: \bar{S}_{A,c} > \hat{\theta}_A} W_{A,c} 
+ \sum_{A: A \subseteq G} \sum_{c: \bar{S}_{A,c} = \hat{\theta}_A} W_{A,c} t_{A,c} \\
+ w_{\text{test}} v_{\text{test}} \mathbf{1}\{X_{n+1} \in G\} = 0 \nonumber
\end{multline}
which implies
$$
\sum_{A:A\subseteq G}\sum_{c=1}^{m_A} W_{A,c} \mathbf{1}\left\{ \bar{S}_{A,c} < \hat{\theta}_{A} \right\} + w_{\text{test}} \mathbf{1}\left\{ S_{n+1} < \hat{\theta}_{n+1} \right\} \mathbf{1}\{X_{n+1} \in G\} + (1-\alpha) \mathcal{T}_{0} = (1-\alpha)D - \mathcal{T},
$$
where
$$
\mathcal{T}_{0} = \sum_{A:A\subseteq G}\sum_{c=1}^{m_A} W_{A,c} \mathbf{1}\left\{ \bar{S}_{A,c} = \hat{\theta}_{A} \right\} + w_{\text{test}} \mathbf{1}\left\{ S_{n+1} = \hat{\theta}_{n+1} \right\} \mathbf{1}\{X_{n+1} \in G\}
$$

$$
\mathcal{T} = \sum_{A: A \subseteq G} \sum_{c=1}^{m_A} W_{A,c} t_{A,c} \mathbf{1}\left\{ \bar{S}_{A,c} = \hat{\theta}_{A} \right\} + w_{\text{test}}t_{\text{test}} \mathbf{1}\{S_{n+1}=\hat{\theta}_{n+1}\} \mathbf{1}\{X_{n+1} \in G\}.
$$

By the definition of $\hat{p}_{A}(t)$ and $\hat{N}$, we have
$$
\hat{N}-\alpha \mathcal{T}_{0} = (1-\alpha)D - \mathcal{T}
$$
Because $t_{A,c}, t_{\text{test}} \in [\alpha-1, \alpha]$, 
$$
\alpha \mathcal{T}_0 - \mathcal{T} \ge 0,
$$
yielding 
$$
\hat{N} \ge (1-\alpha) D.
$$
This yields
$$
\begin{aligned}
\mathrm{P}(Y_{n+1} \in \mathcal{C}(X_{n+1} \mid \widetilde{\mathcal{D}}) \mid X_{n+1} \in G, \mathcal{E}) 
& = \frac{N}{D} \\ 
& \ge \frac{ \hat{N} - \epsilon\sum_{A: A \subseteq G} W_A}{D} \\ 
& \ge \frac{(1-\alpha) D- \epsilon\sum_{A: A \subseteq G} W_A}{D} \\ 
& \ge 1-\alpha - \epsilon,
\end{aligned}
$$
where the last inequality is due to $\sum_{A: A \subseteq G} W_A< D$. Taking expectation on both sides w.r.t. $\mathcal{E}$ yields the desired result.

\section{Group-Conditional Coverage Upper bound of GC-FCP} \label{appendix:upper_bound}
\begin{theorem}[Upper bound coverage in the case of perfect quantile regression] \label{thm:upper_bound}
Assume that the conditional score distribution $S\mid X$ is continuous and fix any confidence level $\eta\in(0,1)$.
Then, with probability at least $1-\eta/2$ over the calibration data, for every group $G\in\mathcal{G}$, GC-FCP satisfies
\begin{equation} \label{eq:upper_bound_gcfcp}
\mathbb{P}(Y_{n+1} \in \mathcal{C}(X_{n+1} \mid \widetilde{\mathcal{D}}) \mid X_{n+1} \in G) \\ 
\leq 1 - \alpha + \frac{\pi}{\delta} + \Delta_G,
\end{equation}
where the additive term
\begin{equation} 
\Delta_G = |\mathcal{G}| \frac{ \max\left\{ \frac{\pi}{\delta}, \, \sum_{k=1}^K \frac{\pi_k}{n_k + 1} \right \} }{\sum_k \left[ \frac{\pi_k n_k}{n_k+1} \left( p_{k,G} - \sqrt{\frac{\log(2K / \eta)}{2n_k}} \right) \right]},
\end{equation}
with $p_{k,G} = \mathrm{P}_{X,k}(X\in G)$ being the probability that a covariate $X$ drawn from client $k$ belongs to group $G$.
\end{theorem}
\textit{Proof:} See Appendix \ref{sec:proof-thm-upper-bound}.

Theorem~\ref{thm:upper_bound} shows that GC-FCP cannot be arbitrarily conservative. The group-wise coverage lies in a narrow band around $1-\alpha$, with band width controlled explicitly by the additive term $\Delta_G$, which consists of the mis-coverage gap $\epsilon$, calibration sizes $n_k$, and the mixture-weighted group mass $p_{k,G}$.
For instance, when the mixture-weighted group mass $\sum_k \pi_k p_{k,G}$ is small, the denominator shrinks, and $\Delta_G$ grows, which suggests that an efficient prediction set for low-support groups requires more calibration samples.

\subsection{Proof of Theorem~\ref{thm:upper_bound}}
\label{sec:proof-thm-upper-bound}
For any $G\in \mathcal{G}$, where $|\mathcal{G}|=d$, we begin with
$$
\begin{aligned}
\mathrm{P}(Y_{n+1} \in \mathcal{C}(X_{n+1} \mid \widetilde{\mathcal{D}}) \mid X_{n+1} \in G, \mathcal{E}) 
 = \frac{N}{D} 
 \le \frac{ \hat{N} + \epsilon\sum_{A: A \subseteq G} W_A}{D} 
 \le \frac{ \hat{N}}{D} + \epsilon 
\end{aligned}
$$
Record that $t_{A,c}, t_{\text{test}} \in [\alpha-1, \alpha]$, which yields
$$
\mathcal{T} \ge (\alpha-1)\mathcal{T}_0, \quad \hat{N} = (1-\alpha)D - \mathcal{T} + \alpha \mathcal{T}_{0} \\ 
\implies \hat{N}\le (1-\alpha)D + \mathcal{T}_{0}
$$
Hence, 
$$
\mathrm{P}(Y_{n+1} \in \mathcal{C}(X_{n+1} \mid \widetilde{\mathcal{D}}) \mid X_{n+1} \in G, \mathcal{E})  \le 1-\alpha+\epsilon + \frac{\mathcal{T}_0}{D}
$$
Now our target is to upper bound
$$
\mathcal{T}_{0} = \sum_{A:A\subseteq G}\sum_{c=1}^{m_A} W_{A,c} \mathbf{1}\left\{ \bar{S}_{A,c} = \hat{\theta}_{A} \right\} + w_{\text{test}} \mathbf{1}\left\{ S_{n+1} = \hat{\theta}_{n+1} \right\} \mathbf{1}\{X_{n+1} \in G\}
$$
under the assumption that the distribution of $S\mid X$ is continuous. 

To bound this term, we first simplify the notation by relabeling the pseudo-points with the test point with $i=1,\ldots,m,m+1$, where $m = \sum_{A \in \mathcal{A}} m_A = \mathcal{O}(\delta |\mathcal{A}|)$ is the total number of clusters and $\bar{S}_{m+1}=S_{n+1}$. 

\begin{claim} \label{claim:no_tie}
    Under the conditions of Theorem \ref{thm:upper_bound}, with probability 1, we have
    $$
    \sum_{A:A\subseteq G}\sum_{c=1}^{m_A}\mathbf{1}\left\{ \bar{S}_{A,c} = \hat{\theta}_{A} \right\} + \mathbf{1}\left\{ S_{n+1} = \hat{\theta}_{n+1} \right\} \leq d.
    $$
\end{claim}
\begin{proof}
    See Appendix~\ref{appendix:proof_of_no_tie_claim}
\end{proof}

By Claim~\ref{claim:no_tie}, with probability 1, we have
$$
\frac{\mathcal{T}_0}{D} \le \frac{d \cdot \max\left\{\max_{A,c:A\subseteq G,c\in [m_A]} W_{A,c}\ ,\ w_{\text{test}}\mathbf{1}\{X_{n+1} \in G\} \right\}}{\sum_{A: A \subseteq G} W_A+ w_{\text{test}} \mathbf{1}\{X_{n+1} \in G\}}
$$
Next, we bound RHS using the properties of T-Digest. Record that $w_{\text{test}} = \sum_{k=1}^K \lambda_k = \sum_{k=1}^K \frac{\pi_k}{n_k + 1}$.

\begin{claim} \label{claim:bounded_weight}
    Assume that GC-FCP uses the scale function $r(q) = \frac{\delta}{2\pi} \arcsin(2q - 1)$ for $q \in [0,1]$, where $q$ is the normalized cumulative weight, i.e., quantile. Then, we have
    $$
        \max_{A,c} W_{A,c} \le W_A\sin\left(\frac{\pi}{\delta}\right) \le \sin\left(\frac{\pi}{\delta}\right) \le \frac{\pi}{\delta}.
    $$
\end{claim}
\begin{proof}
    This is a direct result of Lemma~\ref{lem:tdigest_mass_arcsine} and $W_A \le 1$.
\end{proof}

\begin{claim} \label{claim:lower_bound_sum_weights}
    Assume that the mixture weights $\pi_k > 0$, with $p_{k,G} = P_{X,k}(X\in G) > 0$, for any $\eta>0$, with probability larger than $1-\eta/2$, we have
    \begin{equation}
        \sum_{A: A \subseteq G} W_A\ge \sum_k \frac{\pi_k n_k }{n_k+1} \left(p_{k,G} - \sqrt{ \frac{ \log(2K / \eta) }{2n_k}}\right).
    \end{equation}
\end{claim}
\begin{proof}
    See Appendix~\ref{appendix:proof_of_lower_bound_sum_weights}.
\end{proof}

Combining Claims~\ref{claim:no_tie}, \ref{claim:bounded_weight}, and \ref{claim:lower_bound_sum_weights}, we have
$$
\frac{\mathcal{T}_0}{D} \le d \frac{ \max\left\{ \frac{\pi}{\delta}, \sum_{k=1}^K \frac{\pi_k}{n_k + 1} \right\} }{\sum_k \frac{\pi_k n_k }{n_k+1} \left(p_{k,G} - \sqrt{ \frac{ \log(2K / \eta) }{2n_k}}\right) }
$$
Record that 
$$
\mathrm{P}(Y_{n+1} \in \mathcal{C}(X_{n+1} \mid \widetilde{\mathcal{D}}) \mid X_{n+1} \in G, \mathcal{E})  \le 1-\alpha+\epsilon + \frac{\mathcal{T}_0}{D},
$$
combing the bound of $\mathcal{T}_0/D$ and taking expectation over $\mathcal{E}$ yield the desired result.

\subsection{Proof of Claim~\ref{claim:no_tie}} \label{appendix:proof_of_no_tie_claim}
\begin{proof}
    Let $\mathcal{L}$ denote the set of all pseudo-points, consisting of the clusters across all atoms and the test point. Specifically, $\mathcal{L} = \{(A,c) : A \in \mathcal{A}, c \in [m_A]\} \cup \{\text{test}\}$, with $|\mathcal{L}| = m + 1$, where $m = \sum_{A \in \mathcal{A}} m_A = \mathcal{O}(\delta |\mathcal{A}|)$ is the total number of clusters at merged T-Digests. For each $l \in \mathcal{L}$, define the feature $\Psi_l = \Phi_{A}$ if $l = (A,c)$ is a cluster in atom $A$, or $\Psi_l = \Phi(X_{n+1})$ if $l = \text{test}$, and the score $T_l = \bar{S}_{A,c}$ if $l = (A,c)$, or $T_l = S_{n+1}$ if $l = \text{test}$. 

The optimizer $\hat{\beta}_{S_{n+1}}$ is the solution to the weighted quantile regression problem $\eqref{eq:coreset_beta_augqr}$ with $S=S_{n+1}$. A tie at pseudo-point $l$ occurs if $T_l = \hat{\beta}_{S_{n+1}}^T \Psi_l$. Let $\mathcal{D}^{aug}=\{X_{i,k}\}_{i\in[n_k],k\in[K]}\cup\{X_{n+1}\}$.

As in \citep{gibbs2023conformal}, we calculate the probability
$$
\mathrm{P}\left(\sum_{l \in \mathcal{L}} \mathbf{1}\left\{T_l = \hat{\beta}_S^T \Psi_l\right\}>d\mid \mathcal{D}^{aug} \right)
$$
The event $\left\{\sum_{l \in \mathcal{L}} \mathbf{1}\left\{T_l = \hat{\beta}_S^T \Psi_l\right\}>d\right\}$ implies that there exists a subset $\mathcal{I} \subseteq \mathcal{L}$ with $|\mathcal{I}| = d+1$ such that $T_l = \hat{\beta}_S^T \Psi_l$ for all $l \in \mathcal{I}$. Therefore,
$$
\mathrm{P}\left(\sum_{l \in \mathcal{L}} \mathbf{1}\left\{T_l = \hat{\beta}_S^T \Psi_l\right\}>d\mid \mathcal{D}^{aug} \right) \leq \sum_{\mathcal{I} \subseteq \mathcal{L}, \, |\mathcal{I}| = d+1} \mathbb{P}\left( \exists \beta \in \mathbb{R}^d \text{ s.t. } T_l = \beta^T \Psi_l \ \forall l \in \mathcal{I} \ \middle|\  \mathcal{D}^{aug} \right)
$$
For a fixed $\mathcal{I}$, the event is that the vector $(T_l)_{l \in \mathcal{I}} \in \operatorname{span}\{ (\Psi_l^T)_{l \in \mathcal{I}} \}$, where $\operatorname{span}\{ (\Psi_l^T)_{l \in \mathcal{I}} \}$ is the image of the map $\beta \mapsto (\Psi_l^T \beta)_{l \in \mathcal{I}}$, which is a linear subspace of $\mathbb{R}^{d+1}$ with dimension at most $d$.

Under the continuity assumption, the conditional distribution of $(S_{i,k})_{i,k} \cup \{S_{n+1}\}$ given the $X$'s is absolutely continuous with respect to Lebesgue measure on $\mathbb{R}^{n+1}$.  

The mapping from the original scores ${S_{i,k}}$ (and $S_{n+1}$) to the coreset scores $\{T_l\}_{l \in \mathcal{L}}$ is piecewise affine. For each atom $A$, the score space $\mathbb{R}^{n_A}$ is partitioned into finitely many open cones $\mathcal{R}_\pi$ indexed by permutations $\pi \in \mathcal{S}_{n_A}$, where $\mathcal{R}_\pi = \{ \mathbf{s} \in \mathbb{R}^{n_A} : s_{\pi(1)} < \dots < s_{\pi(n_A)} \}$, where boundaries have measure zero. Within each $\mathcal{R}_\pi$, the sorted scores $\mathbf{s}^{(\pi)}$ determine fixed quantile positions $q_l$ based on permuted weights, leading to fixed cluster assignments via the deterministic T-Digest merge rules. Thus, each $\bar{S}_{A,c}$ is an affine function of $\mathbf{s}^{(\pi)}$ (weighted average over fixed indices), and hence affine in $\mathbf{s}$. The full map to $(T_l)_{l \in \mathcal{I}}$ is therefore piecewise affine across atoms and the test score, preserving measure-zero sets under the continuous distribution of $S \mid X$.

Since the subspace has Lebesgue measure zero in $\mathbb{R}^{d+1}$, its pre-image under the affine map also has measure zero. For a fixed $\mathcal{I}$, we have
$$
\mathrm{P}\left( \exists \beta \in \mathbb{R}^d \text{ s.t. } T_l = \beta^T \Psi_l \ \forall l \in \mathcal{I} \ \middle|\  \mathcal{D}^{aug} \right) = 0.
$$
Since there are at most $\binom{m+1}{d+1}$ many such $\mathcal{I}$, the union bound yields
$$
\mathrm{P}\left(\sum_{l \in \mathcal{L}} \mathbf{1}\left\{T_l = \hat{\beta}_S^T \Psi_l\right\}>d\mid \mathcal{D}^{aug} \right)=0
$$
Because
$$
\left\{ \sum_{A:A\subseteq G}\sum_{c=1}^{m_A}\mathbf{1}\left\{ \bar{S}_{A,c} = \hat{\theta}_{A} \right\} + \mathbf{1}\left\{ S_{n+1} = \hat{\theta}_{n+1} \right\} >d \right\} \subseteq \left\{\sum_{l \in \mathcal{L}} \mathbf{1}\left\{T_l = \hat{\beta}_S^T \Psi_l\right\}>d\right\},
$$
we have
$$
\mathrm{P}\left(\sum_{A:A\subseteq G}\sum_{c=1}^{m_A}\mathbf{1}\left\{ \bar{S}_{A,c} = \hat{\theta}_{A} \right\} + \mathbf{1}\left\{ S_{n+1} = \hat{\theta}_{n+1} \right\} > d \mid \mathcal{D}^{aug} \right) \\ \le \mathbb{P}\left(\sum_{l \in \mathcal{L}} \mathbf{1}\left\{T_l = \hat{\beta}_S^T \Psi_l\right\}>d\mid \mathcal{D}^{aug} \right)=0.
$$
Marginalizing $\mathcal{D}^{aug}$ yields the desired result.
\end{proof}


    

\subsection{Proof of Claim~\ref{claim:lower_bound_sum_weights}} \label{appendix:proof_of_lower_bound_sum_weights}
\begin{proof}
    The calibration weight in $G$ is $\sum_{A: A \subseteq G} W_A= \sum_k \lambda_k \cdot \#\{i : X_{i,k} \in G\}$, where $\lambda_k = \pi_k / (n_k + 1)$.

Let $Z_k = \#\{i : X_{i,k} \in G\}$, then $\sum_{A: A \subseteq G} W_A= \sum_k \lambda_k Z_k$, and 
$$
Z_k \sim \mathrm{Bin}(n_k, p_{k,G}),\text{ with }p_{k,G} = P_{X,k}(X\in G)
$$
and the expected weight is
$$
E\left[\sum_{A: A \subseteq G} W_A\right] = \sum_k \lambda_k n_k p_{k,G} = \sum_k \frac{\pi_k n_k}{n_k + 1} p_{k,G}.
$$
By Hoeffding's inequality, for each $k$,
$$
\mathrm{P}\left( Z_k \le n_k p_{k,G} - t_k \right) \le \exp\left(\frac{-2 t_k^2}{n_k} \right),
$$
for $t_k >0$. Set $t_k = \sqrt{ (n_k /2) \log(2K / \eta) }$ with any $\eta>0$, we have
$$
\mathrm{P}( Z_k \le n_k p_{k,G} - t_k ) \le \frac{\eta}{2K}.
$$
By union bound over $k$, with probability larger than $1 - \eta /2$,
$$
\sum_k \lambda_k Z_k \ge \sum_k \lambda_k (n_k p_{k,G} - t_k) =\sum_k \frac{\pi_k n_k }{n_k+1} \left(p_{k,G} - \sqrt{ \frac{ \log(2K / \eta) }{2n_k}}\right).
$$
We have, with probability larger than $1 - \eta /2$,
$$
\sum_{A: A \subseteq G} W_A=\sum_k \lambda_k Z_k \ge \sum_k \frac{\pi_k n_k }{n_k+1} \left(p_{k,G} - \sqrt{ \frac{ \log(2K / \eta) }{2n_k}}\right).
$$

\end{proof}

\section{Dual construction for (\ref{eq:condcp_set}) under the objective (\ref{eq:gcfcp_augqr})}
\label{sec:dual-construction}
\paragraph{Reformulate the primal:}
Let $\lambda_k = \pi_k / (n_k + 1)$. The primal objective of \eqref{eq:gcfcp_augqr} is rewritten as:
$$
\sum_{k=1}^K \lambda_k \sum_{i=1}^{n_k} \ell_\alpha(g(X_{i,k}), S_{i,k}) + \left( \sum_{k=1}^K \lambda_k \right) \ell_\alpha(g(X_{n+1}), S).
$$
The pinball loss is given by
$
\ell_\alpha(u,s) = (1-\alpha)(s-u)_+ + \alpha(u-s)_+,
$
where $(a)_+ = \max\{a,0\}$.
To reformulate the primal problem~\eqref{eq:gcfcp_augqr}, we first introduce the following claim to rewrite the pinball loss as an optimization problem.
\begin{claim}[Pinball Loss Reformulation] \label{claim:pinball_reformulation}
    Fix any scalar residual $r=s-u$. Consider the following problem
\begin{equation}
\label{eq:problem_pinball_loss}
    \min_{p,q\ge 0}\ (1-\alpha)p+\alpha q
\quad\text{s.t.}\quad r=p-q.
\end{equation}
Then, the optimal value of \eqref{eq:problem_pinball_loss} equals $\ell_\alpha(u,s)$.
\end{claim}
\begin{proof}
Since $r=p-q$, we have $p=r+q$. Together with $p\ge 0$ and $q\ge 0$, feasible pairs satisfy $q\ge \max\{-r,0\}=(-r)_+$ and then $p=r+q\ge 0$.
Substitute $p=r+q$ into the objective:
$$
(1-\alpha)p+\alpha q
=(1-\alpha)(r+q)+\alpha q
=(1-\alpha)r+q.
$$
Thus, minimizing over feasible $q$ is equivalent to
$$
\min_{q\ge (-r)_+}\ (1-\alpha)r + q,
$$
whose minimum is attained at $q^*=(-r)_+$, giving value
$$
(1-\alpha)r + (-r)_+.
$$
If $r\ge 0$, then $(-r)_+=0$, so the value is $(1-\alpha)r=(1-\alpha)(s-u)$. If $r<0$, then $(-r)_+=-r$, so the value is $(1-\alpha)r-r = -\alpha r=\alpha(u-s)$.
Therefore, the optimal value is exactly
$$
(1-\alpha)(s-u)_+ + \alpha(u-s)_+ = \ell_\alpha(u,s).
$$
\end{proof}

By Claim~\ref{claim:pinball_reformulation}, we introduce auxiliary $p_{i,k}, q_{i,k} \geq 0$ for each calibration point $(i,k)$ and $p_k^{\text{test}}, q_k^{\text{test}} \geq 0$ for each virtual test copy per client $k$ and rewrite the problem \eqref{eq:gcfcp_augqr} as
\begin{equation}
\begin{aligned}
\mathrm{(P0)}: \quad \mathop{\min}_{g\in \mathcal{F}_{\mathcal{G}}, p,q}\quad & \sum_{k=1}^K \sum_{i=1}^{n_k} \lambda_k \left( (1-\alpha) p_{i,k} + \alpha q_{i,k} \right) + \sum_{k=1}^K \lambda_k \left( (1-\alpha) p_k^{\text{test}} + \alpha q_k^{\text{test}} \right),\\
\text{s.t.} \quad & S_{i,k} - g(X_{i,k}) = p_{i,k} - q_{i,k}, S - g(X_{n+1}) = p_k^{\text{test}} - q_k^{\text{test}}, \quad  \forall i\in[n_k],~\forall k\in[K] \\ & p_{i,k}, q_{i,k},p_k^{\text{test}},q_k^{\text{test}} \geq 0, \quad \forall i\in[n_k],~\forall k\in[K]
\end{aligned}
\end{equation}


\paragraph{Dual problem:}
Introduce dual variables $\eta_{i,k}$ for calibration constraints and $\eta_k^{\text{test}}$ for each test copy.
Using standard calculations, the dual problem is given by the following linear programming (LP) problem:
\begin{align}
\mathrm{(P1)}: \quad \max_{\{\eta_{i,k}\},\{\eta^{\mathrm{test}}_k\}}~~
& \sum_{k=1}^K\sum_{i=1}^{n_k}\eta_{i,k}S_{i,k} ~+~ S\sum_{k=1}^K \eta^{\mathrm{test}}_k
\label{eq:dual_lp_obj_struct}\\
\text{s.t. }~~
& -\lambda_k\alpha \le \eta_{i,k}\le \lambda_k(1-\alpha), \quad \forall i\in[n_k],~\forall k\in[K], \label{eq:dual_box_cal_struct}\\
& -\lambda_k\alpha \le \eta_k^{\mathrm{test}}\le \lambda_k(1-\alpha), \quad \forall k\in[K], \label{eq:dual_box_test_struct}\\
& \sum_{k=1}^K\sum_{i=1}^{n_k}\eta_{i,k}\Phi(X_{i,k}) ~+~ \sum_{k=1}^K \eta_k^{\mathrm{test}}\Phi(X_{n+1}) = 0.
\label{eq:dual_coupling_struct}
\end{align}
where $\Phi(x) = \bigl(\mathbf{1}\{x \in G\}\bigr)_{G \in \mathcal{G}} \in \{0,1\}^{|\mathcal{G}|}$. The last constraint is to enforce the linearity of $g(x)=\beta^T \Phi(x)$, which arises mathematically from the stationary condition with respect to the primal weights $\beta$.

\paragraph{Set construction via KKT conditions:}
In \citep[Section 4]{gibbs2023conformal}, the conformal prediction set is constructed using the dual solution $\eta_S$ (for input score $S$) and KKT conditions. The optimal $\hat{g}_S(X_{n+1})$ satisfies: $\eta_S^{n+1} = 1-\alpha$ if $S > \hat{g}_S(X_{n+1})$, $\eta_S^{n+1} = -\alpha$ if $S < \hat{g}_S(X_{n+1})$, and $\eta_S^{n+1} \in [-\alpha, 1-\alpha]$ if $S = \hat{g}_S(X_{n+1})$.

Similarly, let $\eta_S^{\text{test}} = \sum_{k=1}^K \eta_{k,S}^{\text{test}}$ be the sum of dual variables over test copies (since the test is duplicated). The set is:
$$
\begin{aligned} \label{eq:dual condCP}
\hat{C}_{\text{dual}}(X_{n+1}) = \left\{ y : \eta_{s(X_{n+1}, y)}^{\text{test}} \le \left( \sum_{k=1}^K \lambda_k \right) (1 - \alpha) \right\},
\end{aligned}
$$
using the weighted upper bound. 

\paragraph{Binary search:} We compute $\hat{C}_{\text{dual}}(X_{n+1})$ using the following two-step procedure. First, using Algorithm 1 of \citep{gibbs2023conformal}, we binary search for the largest value of $S^*$ such that $\eta_{S^*}^{n+1} < 1 - \alpha$. Second, we output all $y$ such that $s(X_{n+1},y) \le S^*$ \citep[Theorem 4]{gibbs2023conformal}.

\section{Additional Experimental Results} \label{appendix:additional_experiments}

\subsection{Experimental Settings and Results of Section~\ref{sec:experiments}} 
\label{appendix:additional_results_of_main}
Other parameters for the synthetic regression can be found in Table~\ref{tab:synth_setup}.

\begin{table}[ht]
\centering
\caption{Experimental setup for synthetic regression.}
\label{tab:synth_setup}

\begin{tabular}{@{}l l@{}}
\toprule
\textbf{Parameter} & \textbf{Value} \\
\midrule
T-Digest parameter $\delta$ & $\delta=250$ \\
Local calibration sizes $n_k$ & $n_1=1000$, $n_k=333$ for $k>1$ \\
Mixture weights $\pi_k$ & $\pi_k=1/K$ for all $k$ \\
Miscoverage level $\alpha$ & $\alpha=0.1$ \\
Model $f(\cdot)$ & linear regression \\
Score $s(x,y)$ & $|y-f(x)|$ \\
Test points $n_{\text{test}}$ & $200$ \\
Monte Carlo runs & $100$ \\
\bottomrule
\end{tabular}
\end{table}

Figure~\ref{fig:cifar_curves} reports the empirical group-wise coverage and the average set size over confidence level $\alpha$ on CIFAR-10. Vanilla CP achieves competitive marginal coverage but exhibits group-wise miscalibration under non-i.i.d. client partitions. GC-FCP reduces these discrepancies, pushing group-wise coverage closer to the target across $G \in \mathcal{G}$. This improvement comes with a moderate increase in set size, reflecting the standard trade-off when enforcing conditional validity over overlapping groups.

Figure~\ref{fig:path_curves} reports the empirical group-wise coverage and the average set size over confidence level $\alpha$ on PathMNIST. As expected, vanilla CP exhibits nonuniform conditional coverage across groups and different values of $\alpha$. GC-FCP improves stability across groups while keeping prediction sets competitive in federated settings where raw examples remain local.

\begin{figure*}[ht]
\centering
\begin{subfigure}[t]{\textwidth}
    \centering
    \includegraphics[width=\linewidth]{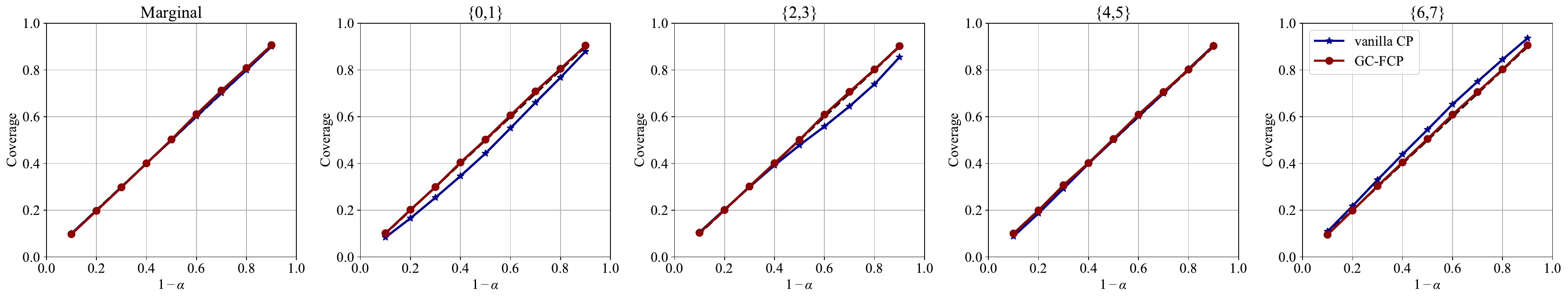}
    \caption{Average coverage over different groups.}
\end{subfigure} \\
\begin{subfigure}[t]{\textwidth}
    \centering
    \includegraphics[width=\linewidth]{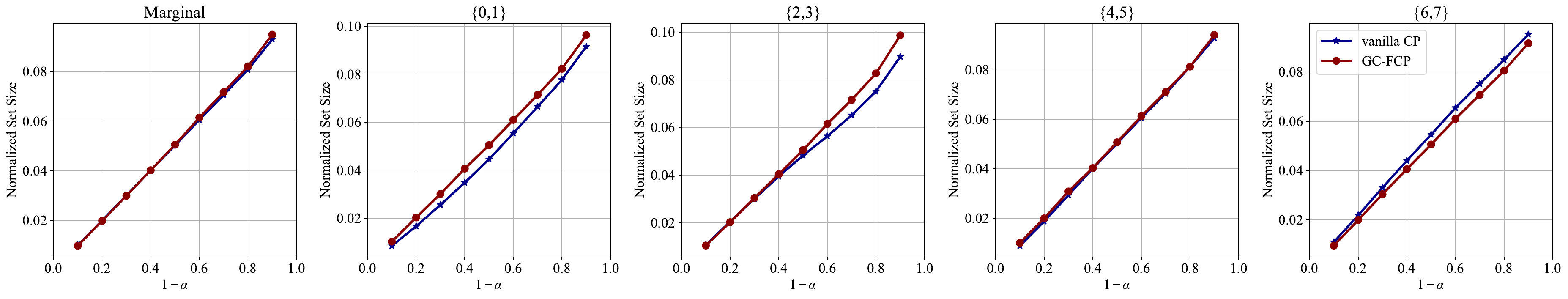}
    \caption{Average set size over different groups.}
\end{subfigure}
\caption{Average coverage and set size versus coverage level $1-\alpha$ with vanilla CP and the proposed GC-FCP on CIFAR-10.}
\label{fig:cifar_curves}
\end{figure*}

\begin{figure*}[ht]
\centering
\begin{subfigure}[t]{\textwidth}
    \centering
    \includegraphics[width=\linewidth]{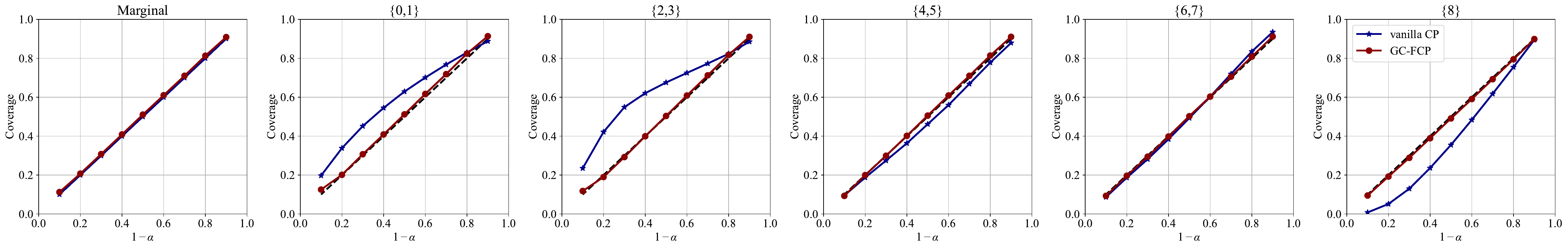}
    \caption{Average coverage over different groups.}
\end{subfigure} \\ 
\begin{subfigure}[t]{\textwidth}
    \centering
    \includegraphics[width=\linewidth]{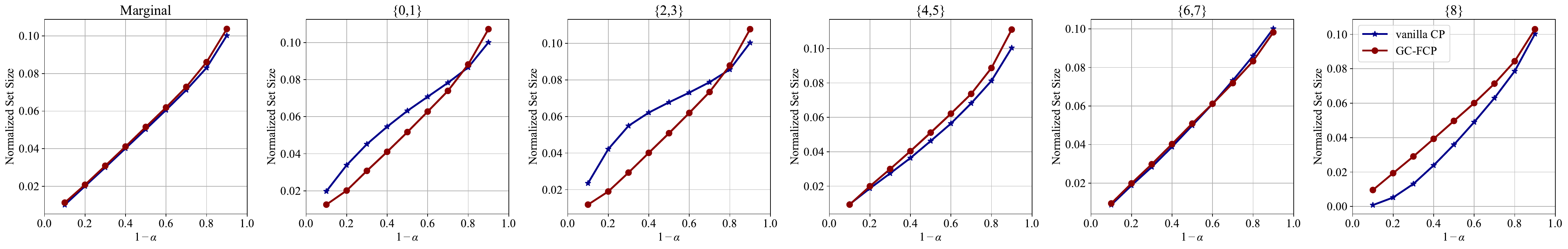}
    \caption{Average set size over different groups.}
\end{subfigure}
\caption{Average coverage and set size versus coverage level $1-\alpha$ with vanilla CP and the proposed GC-FCP on PathMNIST.}
\label{fig:path_curves}
\end{figure*}

\subsection{ImageNet-1K Experiments}
\label{appendix:imagenet_experiment}

{\color{black}
We additionally evaluate GC-FCP on the ImageNet-1K validation set~\citep{deng2009imagenet} using a pretrained ResNet-50 classifier~\citep{He_2016_CVPR}.
In each Monte Carlo repetition, we split the validation data into $n_{\mathrm{cal}}=40{,}000$ calibration samples and $n_{\mathrm{test}}=10{,}000$ evaluation samples.
The calibration set is distributed across $K=50$ clients by a class-wise Dirichlet allocation with concentration parameter $0.3$, enforcing at least $200$ calibration samples per client, and we use uniform mixture weights.
We consider three standard classification scores: THR/HPS~\citep{sadinle2019least}, APS~\citep{romano2020classification}, and RAPS~\citep{angelopoulos2021uncertainty}.
}

{\color{black}
We design 10 semantic groups plus 5 ambiguity groups, which tests coverage across both category and model uncertainty, following ImageNet's semantic hierarchy \citep{deng2009imagenet} and standard confidence-based uncertainty measures in calibrated classification \citep{guo2017calibration}.
Let $\mathcal Y=\{1,\ldots,1000\}$ denote the ImageNet-1K label space, and let $\{\mathcal C_m\}_{m=1}^{10}$ be a fixed partition of labels into ten semantic meta-categories.
For classifier probabilities $p_\theta(\cdot\mid x)$, define
\[
P_m(x)=\sum_{c\in\mathcal C_m}p_\theta(c\mid x),
\qquad
M(x)=\arg\max_{m\in\{1,\ldots,10\}}P_m(x),
\]
which induces semantic groups $G_m^{\mathrm{sem}}=\{x:M(x)=m\}$.
To capture prediction ambiguity, let $p_{(1)}(x)\ge p_{(2)}(x)$ be the two largest class probabilities and define the margin $A(x)=p_{(1)}(x)-p_{(2)}(x)$.
Using calibration-split quantiles of $A(x)$, we form five ambiguity bins $G_b^{\mathrm{amb}}$.
The final family is $\mathcal G=\{G_m^{\mathrm{sem}}\}_{m=1}^{10}\cup\{G_b^{\mathrm{amb}}\}_{b=1}^{5}$.
Each example belongs to one semantic and one ambiguity group, yielding at most $10\times5=50$ semantic--ambiguity atoms after removing empty intersections.

}

{\color{black}
We compare GC-FCP against centralized CP, FedCP~\citep{lu2023federated}, Mondrian-FedCP, a FedCF-style group-wise baseline~\citep{srinivasan2025fedcf}, and importance-weighted FCP~\citep{plassier2023conformal}.
Since the group-wise baselines do not natively produce a simultaneous predictor under overlapping groups, overlapping test points are assigned by a fixed minimum-index rule.
}

\begin{figure*}[t]
\centering
\includegraphics[width=\linewidth]{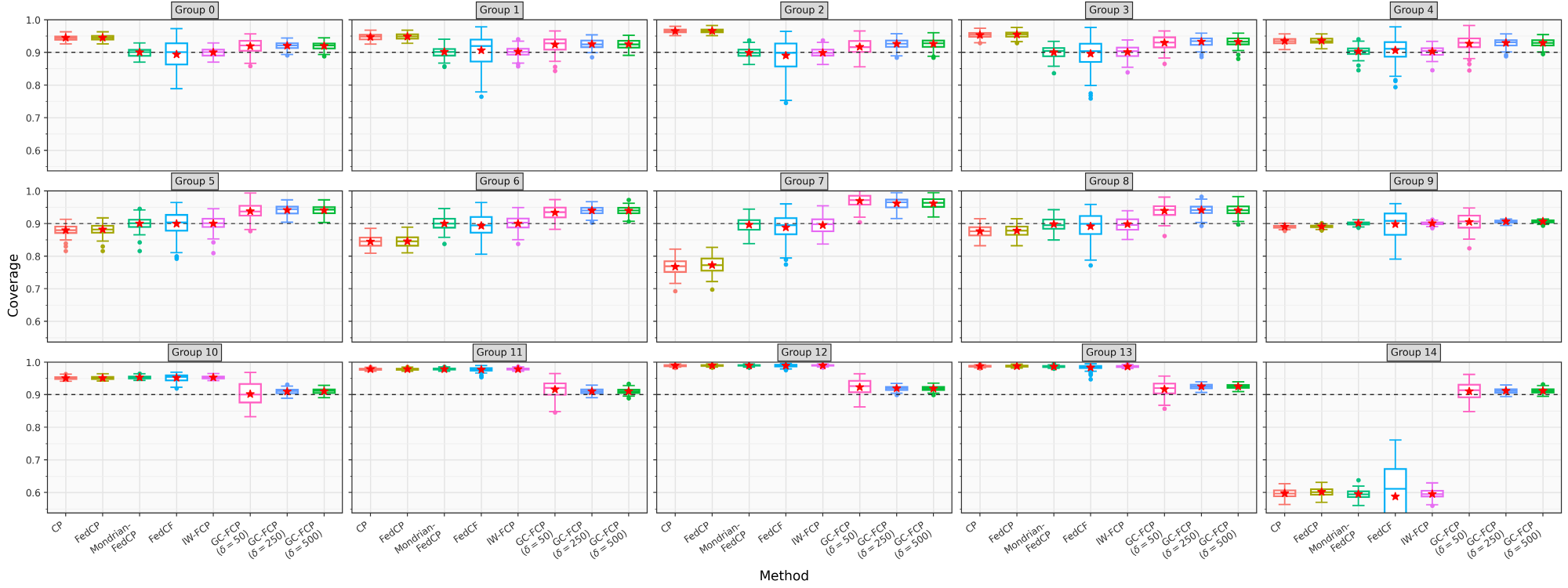}
\caption{\rev{Group-conditional coverage of baselines and GC-FCP on ImageNet-1K with the RAPS score. }}
\label{fig:imagenet_coverage}
\end{figure*}

\begin{figure*}[t]
\centering
\includegraphics[width=\linewidth]{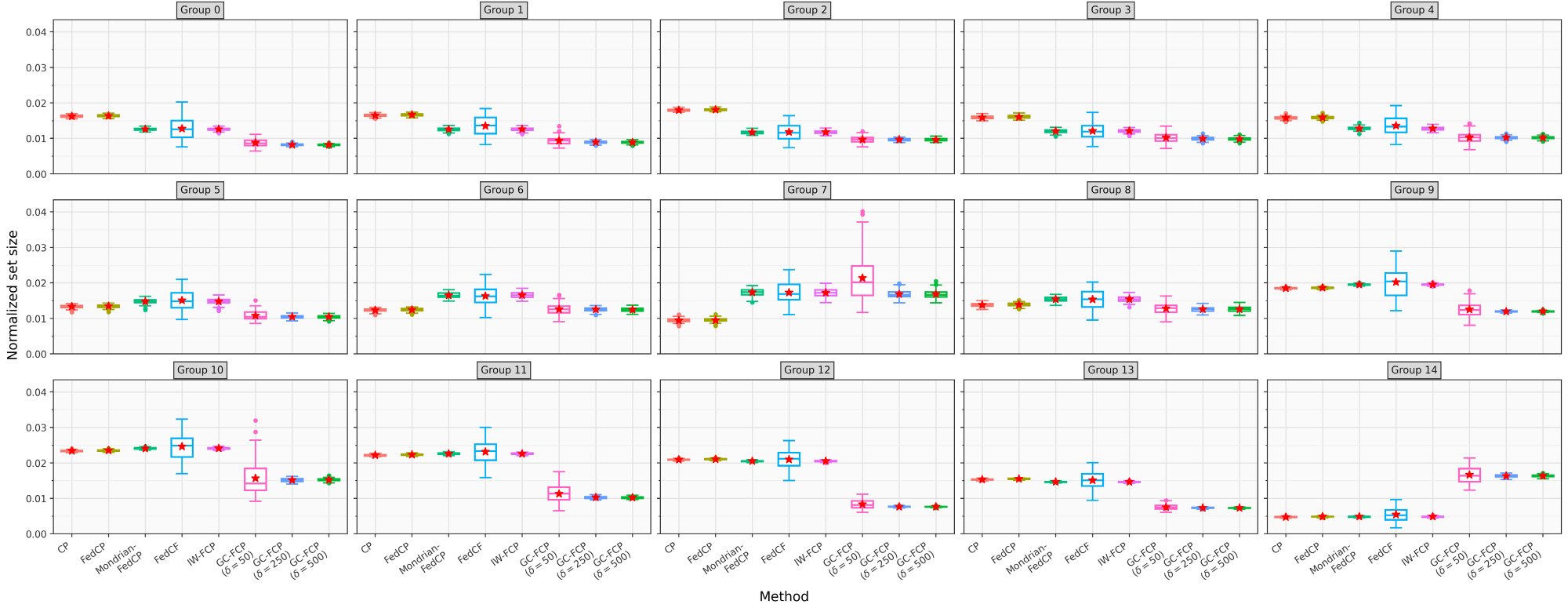}
\caption{\rev{Group-conditional normalized set size of baselines and GC-FCP on ImageNet-1K with the RAPS score.}}
\label{fig:imagenet_set_size}
\end{figure*}

{\color{black}
Figures~\ref{fig:imagenet_coverage} and~\ref{fig:imagenet_set_size} report the main ImageNet comparison using the RAPS score.
GC-FCP with $\delta\in\{50,250,500\}$ achieves near-target coverage across the semantic--ambiguity groups, whereas marginal and group-wise baselines under-cover several difficult groups.
}

\begin{table*}[t]
\centering
\caption{\rev{Ablation of the group design used by GC-FCP on ImageNet-1K under three classification scores. Marginal coverage, average group coverage, and normalized set size are reported as mean $\pm$ standard error over 10 Monte Carlo repetitions; worst-group coverage is computed from pooled group-wise coverages across repetitions.}}
\label{tab:imagenet_group_ablation}
{\color{black}
\begin{tabular}{llcccc}
\hline
\textbf{Score} & \textbf{Group design} & \textbf{Marg. cov.} & \textbf{Worst-group cov.} & \textbf{Avg. group cov.} & \textbf{Norm. set size} \\
\hline
APS & ambiguity & $0.905\pm0.001$ & $0.901$ & $0.905\pm0.001$ & $0.0875\pm0.0008$ \\
APS & semantic & $0.911\pm0.001$ & $0.902$ & $0.927\pm0.001$ & $0.1779\pm0.0015$ \\
APS & semantic--ambiguity & $0.915\pm0.001$ & $0.906$ & $0.926\pm0.001$ & $0.0953\pm0.0005$ \\
RAPS & ambiguity & $0.907\pm0.001$ & $0.902$ & $0.907\pm0.001$ & $0.0110\pm0.0000$ \\
RAPS & semantic & $0.913\pm0.001$ & $0.905$ & $0.927\pm0.002$ & $0.0184\pm0.0001$ \\
RAPS & semantic--ambiguity & $0.916\pm0.001$ & $0.908$ & $0.925\pm0.002$ & $0.0113\pm0.0000$ \\
THR & ambiguity & $0.906\pm0.001$ & $0.900$ & $0.906\pm0.001$ & $0.0023\pm0.0000$ \\
THR & semantic & $0.912\pm0.001$ & $0.903$ & $0.929\pm0.002$ & $0.0017\pm0.0000$ \\
THR & semantic--ambiguity & $0.916\pm0.001$ & $0.907$ & $0.924\pm0.001$ & $0.0096\pm0.0005$ \\
\hline
\end{tabular}
}
\end{table*}

{\color{black}
Table~\ref{tab:imagenet_group_ablation} isolates the effect of group design with $\delta=250$ fixed.
Across THR, APS, and RAPS, semantic, ambiguity, and semantic--ambiguity groups all maintain worst-group coverage near the target, while normalized set size changes with the conditioning family and score.
}

\begin{figure*}[t]
\centering
\includegraphics[width=0.7\linewidth]{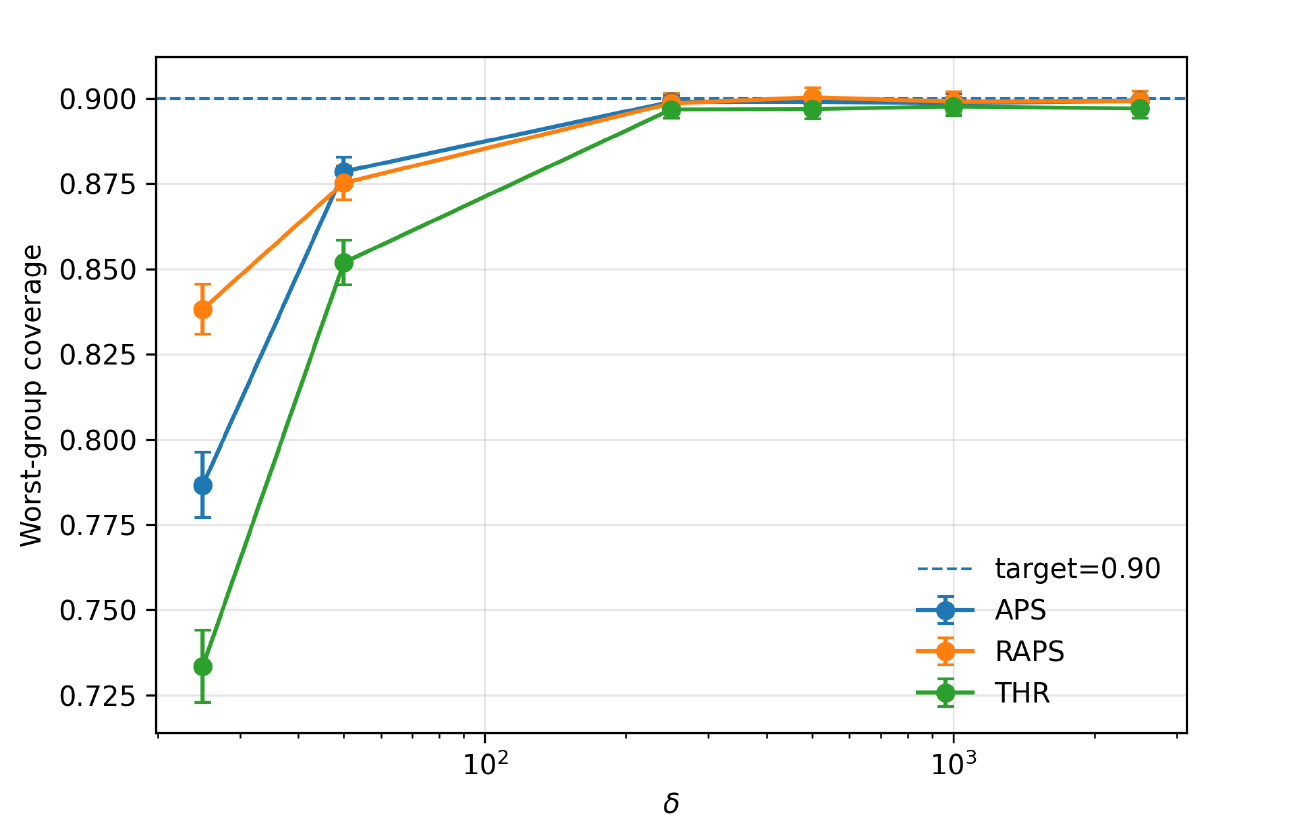}
\caption{\rev{Worst-group coverage of GC-FCP versus T-Digest compression parameter $\delta$ on ImageNet-1K.}}
\label{fig:imagenet_delta_ablation}
\end{figure*}

{\color{black}
Figure~\ref{fig:imagenet_delta_ablation} isolates the compression--accuracy trade-off by sweeping $\delta\in\{25,50,250,500,1000,2500\}$ under the semantic--ambiguity group family.
The results support the theoretical trend in Theorem~\ref{thm:lower_bound}: aggressive compression can reduce worst-group coverage, while moderate $\delta$ values recover the target level.
The effect of federated aggregation is also reflected in the main CIFAR-10 and PathMNIST comparisons between centralized GC-FCP and compressed GC-FCP, where moderate compression closely tracks the exact atom-stratified reference while substantially reducing computation.
}

\end{document}